%% file: index.tex
\def\year{2021}\relax
\documentclass[letterpaper]{article}
\usepackage[utf8]{inputenc}
\usepackage{aaai21}  
\usepackage{times}  
\usepackage{helvet}
\usepackage{courier}
\usepackage[hyphens]{url}
\usepackage{graphicx}
\usepackage{natbib}
\usepackage{caption}
\usepackage{amsthm}
\usepackage{subcaption}
\usepackage{dsfont}
\usepackage[ruled,algo2e,noend]{algorithm2e} 
\usepackage{epstopdf}
\usepackage{mathtools}
\usepackage{amsmath,amsfonts,amssymb}
\usepackage{hyperref}
\usepackage{nameref}
\usepackage{xcolor}
\usepackage{paralist}
\usepackage{microtype}
\usepackage[compact]{titlesec}
\usepackage{tabularx}
\usepackage{textcomp}
\usepackage{float}
\usepackage[cal=cm]{mathalfa}
\usepackage{semantic}
\usepackage{booktabs}
\usepackage{multirow}
\usepackage{subcaption}

\vspace{-1mm}

\newenvironment{Method}[1][htb]
  {% Update algorithm name
   \begin{algorithm2e}[#1]%
  }{\end{algorithm2e}}

\newtheorem{theorem}{Theorem}

\newcommand{\norm}[1]{\left\lVert #1 \right\rVert}
\newcommand{\figwidth}{.9}

\newcommand{\bannerfig}{0.45}

\newcommand{\threesubfig}{.37}

\definecolor{macorchid}{HTML}{7A81FF}
\definecolor{macgrey}{HTML}{929292}
\definecolor{macpurple}{HTML}{663399}
\definecolor{rossocorsa}{rgb}{0.83, 0.0, 0.0}
\definecolor{darkpastelred}{rgb}{0.76, 0.23, 0.13}

\DeclareMathOperator*{\E}{\mathbb{E}}

\newcommand{\clname}{Chung-Lu}
\newcommand{\er}{\texttt{ER}}
\newcommand{\pa}{\texttt{PA}}

\newcommand{\cl}{\texttt{CL}}
\newcommand{\hd}{high-degree}
\newcommand{\ld}{low-degree}

\newcommand{\baname}{Barabási–Albert}
\DeclareMathOperator*{\argmin}{argmin}
\DeclareMathOperator*{\argmax}{argmax}

\SetKwInput{KwGiven}{Given}
\SetKwInput{KwFind}{Find}
\SetKwInput{KwWhere}{Where}
\SetKwInput{KwMin}{Minimizing}

\title{GAEA: Graph Augmentation for Equitable Access via Reinforcement Learning}
\author {
        Govardana Sachithanandam Ramachandran\textsuperscript{\rm 1},
        Ivan Brugere\textsuperscript{\rm 2} \thanks{I.~Brugere and L.~R.~Varshney completed this work while they were with Salesforce Research.},
        Lav R.\ Varshney\textsuperscript{\rm 3*} , and 
        Caiming Xiong\textsuperscript{\rm 1}  \\ 
}

\affiliations {
    \textsuperscript{\rm 1} Salesforce Research \\
    \textsuperscript{\rm 2} University of Illinois at Chicago \\
    \textsuperscript{\rm 3} University of Illinois at Urbana-Champaign \\
    gramachandran@salesforce.com, ivan@ivanbrugere.com, varshney@illinois.edu, cxiong@salesforce.com
}

\begin{document}

\nocopyright
\maketitle

\begin{abstract}
Disparate access to resources by different subpopulations is a prevalent issue in societal and sociotechnical networks. For example, urban infrastructure networks may enable certain racial groups to more easily access resources such as high-quality schools, grocery stores, and polling places. Similarly, social networks within universities and organizations may enable certain groups to more easily access people with valuable information or influence. Here we introduce a new class of problems, Graph Augmentation for Equitable Access (GAEA), to enhance equity in networked systems by editing graph edges under budget constraints.  We prove such problems are NP-hard, and cannot be approximated within a factor of $(1-\tfrac{1}{3e})$. We develop a principled, sample- and time- efficient Markov Reward Process (MRP)-based mechanism design framework for GAEA. Our algorithm outperforms baselines on a diverse set of synthetic graphs. We further demonstrate the method on real-world networks, by merging public census, school, and transportation datasets for the city of Chicago and applying our algorithm to find human-interpretable edits to the bus network that enhance equitable access to high-quality schools across racial groups. Further experiments on Facebook networks of universities yield sets of new social connections that would increase equitable access to certain attributed nodes across gender groups.

\end{abstract}
\input{intro}

\input{related}

\input{problem}
\input{model}

\input{synthetic}

\input{results}

\input{appendix}
%\bibliography{sample-base}

\end{document}

%% file: intro.tex
\section{Introduction}

Designing systems and infrastructures to enable equitable access to resources has been a longstanding problem in economics, planning, and public policy. A classical setting where equity is a strong design criterion is facility location for public goods such as schools, grocery stores, and voting booths \cite{MumphreySW1971,McAllister1976, Mandell1991}.  Since individual-level fairness is often impossible for these spatially-structured problems, the focus has been on group-level fairness; several metrics have been proposed and optimized \cite{MarshS1994}, including simultaneous optimization of several metrics \cite{GuptaJRYZ2020}.  In this paper, we take up the same challenge of equitable access to resources, but for network-structured rather than spatially-structured problems.  We specifically consider editing the edges of graphs under budget constraints to improve equity in group-level access, under a diffusion model of mobility dynamics.  We use the \textit{demographic parity} metric \cite{8452913}.  In economics, a common measure of (in)equity called the \emph{Gini index} is often used to characterize the statistical dispersion in income or wealth, but also access to resources at individual or group levels \cite{ColeBCN2018,MalakarMP2018}, which we also measure.

Facility location problems (not considering equity) have been well-studied in network settings for applications in epidemiology \cite{7393962, Zhang2019}, surveillance \cite{krause_leskovec}, and influence maximization  \cite{im_survey, 10.1145/956750.956769}. These involve changing node properties of a graph; contrarily we edit the edges of graphs.  
Formally we prove our graph editing problem---Graph Augmentation for Equitable Access (GAEA)---is a generalization of facility location and that it is NP-hard.  

In urban transportation networks,  different racial, ethnic, or socioeconomic groups may have varying access to high-quality schools, libraries, grocery stores, and voting booths, which in turn lead to disparate health and educational outcomes, and political power. Interventions to enhance equity in transportation networks may include increasing public transit for the most effective paths to resources such as high-quality schools. From an implementation perspective, changing bus routes may be easier than relocating schools. We do not consider strategic aspects of congestion in transportation networks \cite{YounGJ2008}.

In social networks within organizations like schools, people in different racial or gender groups may have varying access to specific people that hold valuable information or have significant influence, whom we cast as network resources.  This in turn may lead to disparate outcomes in social life.  Interventions to enhance equity include encouraging friendship between specific individuals in the social network.  This can be done, e.g.\ in university settings by offering free meals for two specific people to meet: what we call \emph{buddy lunches}. 

Since AI models are  embedded in numerous consequential sociotechnical systems, ensuring equity in their operation has emerged as a fundamental challenge  \cite{dwork2018fairness,mehrabi2019survey}.  Of more relevance to us, however, AI techniques are being used to design equitable public policies e.g.\ for taxation \cite{Zheng_ea2020}.  Here we use AI methods for system/infrastructure redesign to increase equitable access to resources for settings such as transportation networks \cite{HayT1991,BowermanHC1995} and social networks \cite{AbebeIKL2019, FishBBFSV2019}.  In particular, we develop a  Markov  Reward  Process  (MRP) framework and a principled, sample- and time-efficient reinforcement learning technique for the GAEA optimization.\footnote{Code, and data: \url{https://github.com/salesforce/GAEA}} Our approach produces \textit{interpretable}, localized graph edits and outperforms deterministic baselines on both synthetic and real-world networks.

\begin{figure}[t]
      \centering
  \begin{subfigure}{\bannerfig\columnwidth}

  \includegraphics[width=\columnwidth]{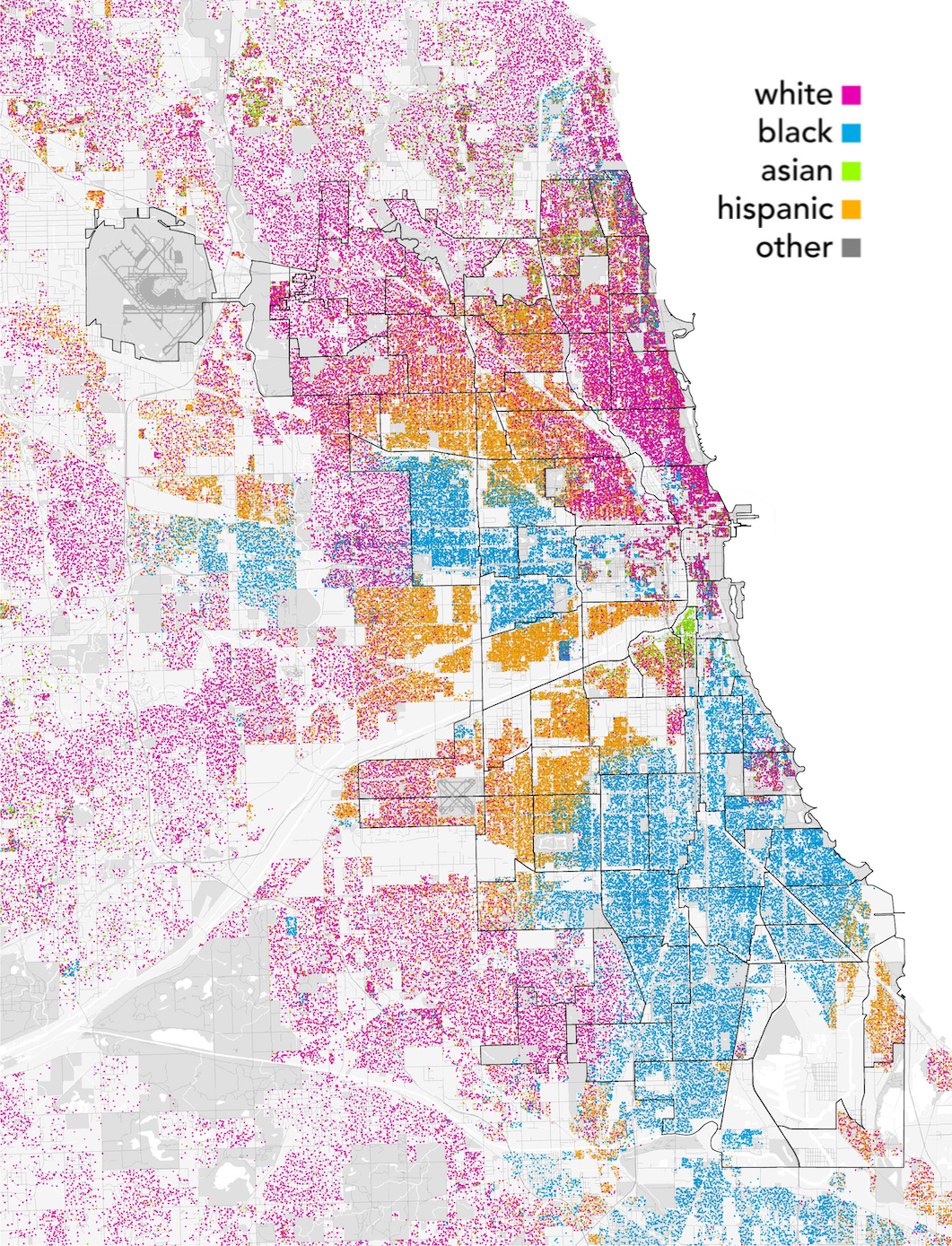}
  \caption{Chicago demographics by race/ethnicity}
  \label{subfig:demo_1}
  \end{subfigure}
   \begin{subfigure}{\bannerfig\columnwidth}
    \centering
        \includegraphics[width=\columnwidth]{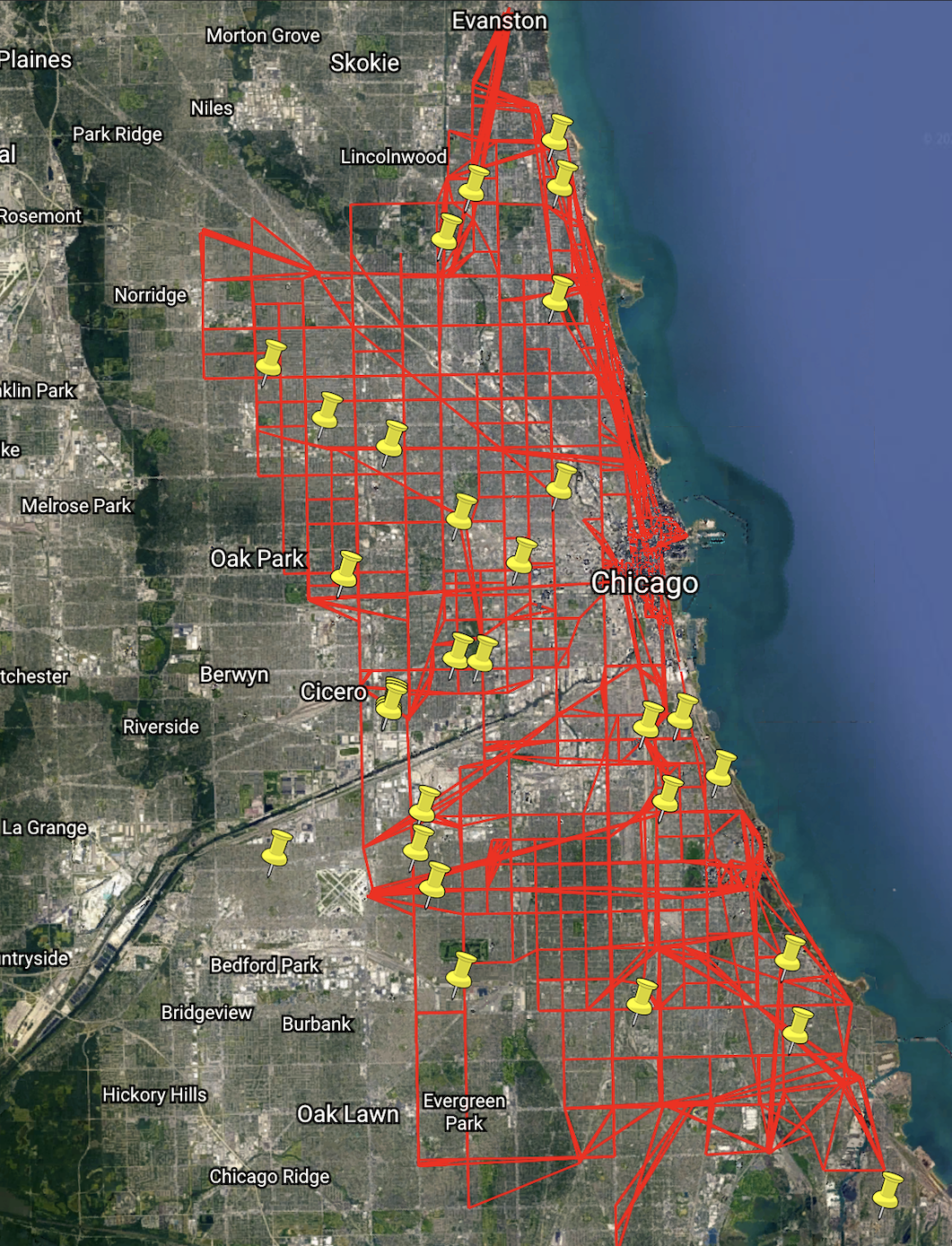}
  \caption{Chicago transit network and school locations}
  \label{subfig:demo2}
  \end{subfigure} 
  \caption{Chicago demographics and infrastructure, (a) shows demographics, demonstrating highly segregated areas of the city by race and ethnicity; (b) shows a transit network (red) we collected for this work, induced from Chicago Transit Authority bus routes. We also show (yellow) the location of schools within our dataset from the Chicago Public Schools.}\label{fig:demo_net}
\end{figure}
%\subsection{Equitable School Access in Chicago}
%In Section \ref{subsec:chicago} 
We demonstrate the effectiveness of our approach on real-world transportation network data from the City of Chicago. Chicago is the most spatially segregated city in the United States via 2010 census data \cite{logan2014diversity}. At the same time, prior work shows this segregation yields significant disparity in education \cite{groeger2018miseducation} and health \cite{doi:10.2105/AJPH.2009.165407} outcomes by race and ethnicity, particularly among White, Black, and Hispanic communities. In the Chicago Public School District $299$---constituting the entire city of Chicago---White students are the equivalent of $3$ academic years ahead of Black students, and $2.3$ years ahead of Hispanic students. Our technique yields an edited graph which reduces disparity in physical access to high-quality schools.

We also present a similar analysis for gender in university social networks using the node-attributed Facebook100 network data \cite{TraudMP2012}.
%Due to homophily in social networks in universities, as captured in the node-attributed Facebook100 network data \cite{TraudMP2012}, there is disparate access for gender groups to, say, upperclassmen in particular male-majority majors who may have distinct information.  Unequal access to these ``reward'' people may perpetuate later glass ceilings.  Our technique yields an edited graph which reduces in access to these reward nodes. 

%Our method is interpretable and 

%under sparse rewards that groups receive within 

%the graph editing problem on a diffusion process problem formulation which represents this set of approaches as a graph editing problem

\subsection{Contributions}

Our main contributions are as follows. 
%In this work, we formulate a Constrained Equitable Graph Editing problem and propose an initial model to solving the problem. Our proposed model further generalizes to the related facility placement problem. Concretely, our main contributions include:
\begin{itemize}
    \item The novel Graph Augmentation for Equitable Access (GAEA) problem, which generalizes facility location, and which we prove is NP-hard.
    \item Markov Reward Process framework to address this difficult optimization problem, yielding  principled, sample- and time-efficient  mechanism design  based on reinforcement learning.
    \item Demonstration of efficacy in both synthetic and real-world networks (transportation and social), showing performance improvement over deterministic baselines. 
\end{itemize}
%\begin{itemize}
%    \item We formulate the Constrained Equitable Graph Editing Problem (CEGE) and providing a deterministic, heuristic model for solving it
%    \item We propose a machine learning model which makes sparse, interpretable edits to graph structure under admissible budget and fairness constraints
%    \item We empirically analyze the problem on a diverse set of synthetic graph models 

%\end{itemize}

%% file: related.tex
\subsection{Related Work}
\phantomsection

\subsubsection{Equity in AI}
Recent work on fairness in machine learning has expanded very rapidly. \citet{mehrabi2019survey} outline $23$ definitions of bias introduced from underlying AI models, $6$ definitions of discrimination (i.e.\ the prejudicial and disparate impacts of bias), and $10$ definitions of fairness which mitigate these impacts. There is a similar zoo of definitions for equity in facility location \cite{MarshS1994}.  
We focus here on demographic parity. 
%Prior work focuses primarily on individual fairness \cite{dwork2018fairness} (i.e. producing similar predictions for similar individuals), or group fairness \cite{dwork2018fairness} (i.e. produce similar outcomes for groups of individuals according to some definition). Our work measures \textit{group fairness}, specifically measuring \textit{demographic parity} \cite{8452913}.

%Prior work has extended such group-fairness to e.g.\ subgroup fairness, which defines subsets of groups that have unfair outcomes due to  heterogeneity \textit{within} the group \cite{pmlr-v80-kearns18a,wang2020fairness}.   \citet{10.1145/3292500.3330745} compare across groups in a more challenging ranking problem, to determine whether one group is systematically favored in a recommendation system. We do not consider such within-group variation.

%The study of equity in network-structured problems (besides facility location) is very limited. \citet{bose2019compositional} measure and mitigate bias within the embedding spaces produced on graphs to ensure that downstream tasks on these graph neural network (GNN) embeddings are invariant to protected group attributes. However, %As far as we know, there is no prior work on how network structure introduces inequity for routing and resource allocation problems.
This is the first work measuring and mitigating inequity within an arbitrarily-structured graph environment.

\subsubsection{Graph Diffusion, Augmentation, and Reinforcement Learning}
Our formulation of equitable access lies within a larger body of work on diffusion in graphs. %: the propagation of entities over the edges in a graph over time. 
Prior work has examined  event detection time in sensor networks \cite{krause_leskovec}, or influence maximization in social networks \cite{im_survey, 10.1145/956750.956769}. The problem can also be viewed as maximization of graph connectivity, maximization of spectral radius, or maximization of closeness centrality of reward nodes \cite{mci}.

Graph augmentation is a class of combinatorial graph problems to identify the minimum-sized set of edges to add to a graph to satisfy some connectivity property, such as $k$-connectivity or strong connectedness \cite{doi:10.1137/0222056}: our objective is not such a straightforward combinatorial property. 

Prior work in graph reinforcement learning primarily focuses on coordinating agents with respect to local structure for cooperative and/or competitive tasks \cite{NIPS2019_9024}.  Our work differs in that we learn a system design policy rather than coordination among active agents.

%% file: problem.tex
\section{Problem Definition}
\begin{figure}[t]
      \centering
  \begin{subfigure}{\bannerfig\columnwidth}\includegraphics[width=\columnwidth]{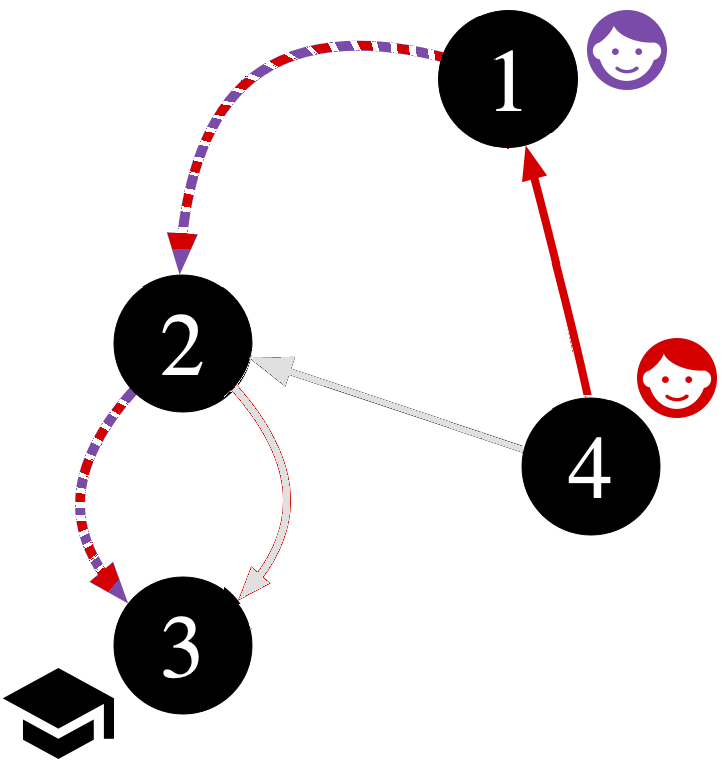}
  \label{subfig:schematic_1}
  \end{subfigure}
   \begin{subfigure}{\bannerfig\columnwidth}\includegraphics[width=\columnwidth]{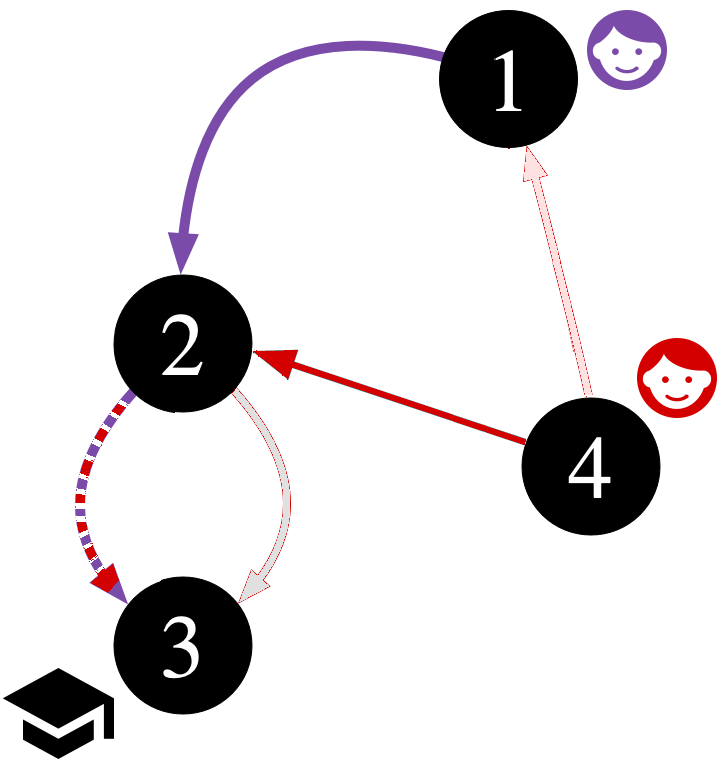}
  \label{subfig:schematic_2}
  \end{subfigure} 
  \caption{A before and after equitable graph augmentation of a toy schematic showing school access between two disparate populations distributed spatially different from each other. Individuals from the group traverse along the edges of their respective color. }\label{fig:schematic}
\end{figure}
\subsection{Toy Example}
Figure \ref{fig:schematic} shows a toy example of our problem. Assume we have a fixed topology where new edges can form, shown in the figure as gray arrows. We have two distinct population of purple and red individuals distributed differently over the nodes of the graph. Based on the graph edges, each group has different levels of access to the same resources, in this case access to schools. The individuals of purple population can access the school at node 3 in two hops 1-2-3, whereas individuals of red population have to traverse one extra hop i.e.\ 4-1-2-3 to access the same resource.

Suppose we have a budget to edit one edge, forming the edge 4-2 will match the red and purple population's access to school. If we have one more edge in the budget, augmenting the additional edge 2-3 will improve access for both populations equitably. This sequence of edge augmentations that improve resource access for the overall population and at the same time maintains equity among subpopulations is what this work strives to achieve. We restrict edge formation to a predefined topology, since arbitrary edge formations are not feasible in many real sociotechnical graphs, e.g.\ in spatial graphs, an edge between two nodes that are arbitrarily far apart cannot be formed.
%In this example, we changed both edges by some constant amount. %This corresponds to many application-specific interventions that change infrastructure that can be  utilized by both groups. Changing only one group's weight may correspond to an \textit{exclusionary} intervention. 
%We have a budget on the \textit{number} of interventions, rather than total magnitude of interventions (which may be well-motivated in some applications but which we do not consider here since editing many edges a small amount may result in many non-trivial external costs such as infrastructure construction across many locations).
\subsection{Preliminaries}

Let graph $\mathbb{G} = (V, E)$ have vertex set $V = \{v_1,\dots,v_n\}$ of size $n$ and edge set $E = \{e_{i,j}\}$ of size $m$. Edges are  unweighted and directed. Let $G$ be a set of groups such as racial or gender groups. Let \textit{reward nodes} be a subset of nodes $R \subseteq V$.

We consider a particle $p_g \in V$ of group $g\in G$ as an instance of starting node positions sampled from a group-specific distribution $\mu_0(g)$. %\footnote{Please note that $\mathbb{G}$ is a graph and $G$ is a set of groups.} 
Letting $d(p_g,r)$ be the shortest path for particle $p_g$ along edges $E$ of $\mathbb{G}$ to reach a reward node in $R$, we define a \textit{utility function} for each group as: 
\begin{equation}
{
    u_g(p_g;E) =  {\E}_{{p_g \sim \mu(g)}} \tfrac{1}{d(p_g,r)} \mbox{.}}
\end{equation}
The utility function is parameterized by the edge set $E$. For simplicity, we refer to $u_g(p_g;E)$ as just $u_g$ and define the utility of the entire group $g$ as:
\begin{equation}{
U_g = {\E}_{{{p_g\sim\mu(g)}}}[u_g] \mbox{.}}
\end{equation}
We define $\overline{U}_G = \sum_{c \in G}U_g/|G|$ as mean utility of all groups and inequity as deviation of group utilities from the mean. i.e $ \sum_{g \in G}{|U_g - \overline{U}_G|}$. To minimize inequity is to minimize this difference. Finally, let a graph augmentation function for a budget $B$ be defined as:  
\begin{equation}{
e :\mathbb{G}, B \rightarrow \mathbb{G}'} 
\end{equation}
where $\mathbb{G}'=(V,E\cup E^u)$, $E^u$ is the edge augmentation to the graph $\mathbb{G}$ constrained by budget $B$, under Hamming distance  $\mathcal{D}(\mathbb{G} - \mathbb{G}') < B$.
%and for notational convenience:

%be a set of weighted random  For a distribution over nodes $mathcal{d}(V)$, sample size $n$ and time $t_{max}$, let a particle process on $W$ process be represented by $\mathcal{P}(\mathcal{d}_0, n, t_{max})$ a set of \textit{particles} be represented as $P$, where the state of each particle is a non-unique node: $P =\{p_k \sim V\}$. Similarly, 

%Let a \textit{transition} function on $G$ be defined as $\mathcal{T}(p,W) \rightarrow v_j$. This represents the new state of node $p$. Let the \textit{value} function $\nu$(p) evaluate the value of a particle in its current state. 
\section{Methods}
Now let us formalize our problem definition.
\DontPrintSemicolon
\begin{algorithm2e}%[H]
  
  \SetAlgorithmName{Problem}
  \KwGiven{A graph $\mathbb{G}$, budget B}
  \KwFind{$\mathbb{G}' = e(\mathbb{G}, B)$}
  \KwWhere{
    $\mathbb{G}^\prime  = \argmax\limits_{\substack{s.t. \sum_{g \in G}{|U_g - \overline{U}_G|} = 0 \\ {{\mathcal{D}}(\mathbb{G} - \mathbb{G}') < B}}}{\E}_{{g\in G}}[U_g]$}
  \caption{{Graph Augmentation for Equitable Access}}
  \label{algorithm:cege}
\end{algorithm2e}
%$\argmax_{G'}\sum_{G_k \in G_{\{m\}}}{\mathbb{E}[\nu(\mathcal{P}(\mathcal{D}_0, W'))]}$ and $G$
% \begin{equation}    
%       U_g = \E_{p \sim \mu(g)}[u_p] \\
% \end{equation}
% \begin{equation}    
% J  = \max\limits_\substack{st. \sum_{g \in G} U_g - \overline{U_G} = 0 \\ \sum_{g \in G} \norm{W_g}_0 < B} \sum_{g \in G}}\E[U_g] \\
% \label{J_d}
% \\end{equation}
%  The constraints are non-differential, especially the number of edges to edit and can not be solved directly as an optimization problem.

%% file: model.tex
\subsection{Greedy Baseline}
For prior baseline, to the best of our knowledge we are not aware of earlier works that does budgeted discrete equitable graph augmentation that maximize utility of disparate groups. That said, we observe that without equitable access across groups, optimizing for maximum utility, $U_g$ alone reduces the problem to maximizing centrality of the reward nodes $r \in R$. Hence we extend the greedy improvement for any monotone submodular $U_g$ proposed by \cite{mci} to the equitable group access setting, which we call Greedy Equitable Centrality Improvement (GECI) baseline. For a given edge set $E$, we define neighborhood $N_g(E)$ as all nodes $u\in V$ and $u \not \in E$ that are in the shortest path to reward node $r$ less than $T$ steps for candidate from group $g$. With budget $B$, the GECI method is defined in Method \ref{algorithm:greedyMCI}.
\begin{Method}
    
    \textbf{Input: }{A directed graph $\mathbb{G} = (V,E)$, neighborhood function $N_g$, and budget $B$} \;
    $E^u := \emptyset $\;
     \For{ $b = 1,2,\dots,B$}{
        $E^e := E\cup E^u$
        $g_{\min} := \argmin{\{U_g(E^e)|g\in G)\}}$\;
        \For{ $u \in V | u \in N_{g_{min}}(E^e) $}{
            \For{$v \in V | v \in E^e $}{
                Compute $U_{g_{\min}}(E^e \cup \{(u,v)\})$
            }
        }
        $u_{\max},v_{\max} := \argmax{\{U_{g_{\min}}(E^e\cup \{(u,v)\})\}}$ \;
        $E^u := E^u \cup \{(u_{\max},v_{\max})\} $
     }
     \Return $E^u$
     \caption{{ Greedy equitable centrality improvement}}
     \label{algorithm:greedyMCI}
\end{Method}

\normalsize
For every augmentation of edge $E^u$, we pick the group $g_{\min}$ that is most disadvantageous or in other words the group with least group utility $U_g$. We then pick a pair of nodes $(u_{\max},v_{\max})$ to form an edge augmentation, such that nodes $v_{\max}$ and $u_{\max}$ are in $E$ and in the neighborhood $N_{g_{min}}$ respectively and result in maximum change in most disadvantageous group's utility $U_{g_{\min}}$, for the candidate group $g_{\min}$, We update the edge augmentation set $E^u := E^u\cup \{(u_{\max},v_{\max})\}$. The graph augmentation step is repeated until the budget $B$ is exhausted.  

\subsection{Optimization Formulation}
Let us consider the GAEA problem from an optimization perspective. Let $U_g$ be the expected utility of a group. Then Pareto optimality of utilities of all groups can be framed as: 
\begin{equation}{ 
J  = \max\limits_{\substack{s.t. \sum_{g \in G}{|U_g - \overline{U}_G|} = 0 \\ \sum_{g \in G}{\norm{E^u}_0 < B}}}{\E}_{{g\in G}}[U_g] \label{J_d}} \mbox{.}
\end{equation}
The first constraint in \eqref{J_d} is the equity constraint while the second is the budget constraint. The constraints are non-differentiable, especially the number of discrete edges to edit, and cannot be solved directly using typical algorithms. Hence we develop our learning approach.

\subsection{Equitable Mechanism Design in MRP}

We frame the graph, $\mathbb{G} = (V,E)$, and dynamic process of reaching reward nodes by particles of different group, $g \in G$, as mechanism design of a finite-horizon Markov Reward Process (MRP). The MRP consists of a finite set of
states, $S$; dynamics of the system defined by a set of Markov state transition probability $P$; a reward function, $R \in \mathbb{R}^{|S|}$; and a horizon defined by the maximum time step $T$ reachable in a random walk. Here states $S$ corresponds to vertices $V$ of the graph $\mathbb{G}$, and the transition dynamics are parameterized by $P = D^{-1}E$, with diagonal matrix $D(i,i)=\sum_{j}E_{(i,j)}$. Unlike Markov Decision Processes (MDP), MRP does not optimize for a policy, instead optimizing for dynamics that maximize the state value function.

The state value function of the MRP for a particle, $p_g$ of group $g$, spawned at state $s_0$ is given by:
\begin{equation}
{
v_g(s_0) = \sum_{t = 0}^{T-1} \gamma^tRP^ts_0 \mbox{,}}
\end{equation}
where $\gamma \in [0,1]$ is the discount factor. This $\gamma$ encourages the learning system to choose shorter paths reachable under the horizon $T$. The expected value function for the group, $g$, is:
\begin{equation}
\label{eqn_Vg}{
V^g = {\E}_{{s_0\sim\mu(g)}}[v^g(s_0)] \mbox{.}}
\end{equation}

We parameterize the transition probability as 
$
P = D^{-1}E^e
$,
where 
$E^e = E +  A \odot E^u$
and diagonal matrix $D$ satisfies $D(i,i)=\sum_{j}E^e(i,j)$.

Here $E \in \{0,1\}^{|S|\times|S|}$ is the edge-adjacency matrix of the unedited original graph $\mathbb{G}$. Further, $A \in \{0,1\}^{|S|\times|S|}$ is a mask adjacency matrix corresponding to given topology, used to restrict the candidate edges for edit. This restriction is useful for spatial graphs where very distant nodes cannot realistically form an edge.  The set $E^u \in \{0,1\}^{|S|\times|S|}$ is the learned discrete choice of edges that are augmented. To make discrete edge augmentation $E^u$ differentiable, we perform continuous relaxation, with reparameterization trick using Gumbel sigmoid \cite{jang2016categorical}, defined by 
\begin{equation}
\label{eqn_Eg}{
E^u{(i,j)} = \frac{1}{(1+\exp(-(\phi(\Vec{0}) + g_{i})/\tau)}  ,\forall i,j \in S}
\end{equation}
where, $g_i = -\log(-\log(\mathbb{U}))$ and $\mathbb{U}$ are the Gumbel noise and uniform random noise respectively. Over the period of training, the temperature $\tau$ is attenuated. As $\tau \to 0 $, $E_g^u$ becomes discrete. Hence we gradually attenuate $\tau \gets \tau * \nu$ at every epoch with a decay factor $\nu$ . Note that the function $\phi(\cdot)$ in \eqref{eqn_Eg} takes the zero vector as input, which  effectively forces the function to learn only the bias term and makes the choice of edits independent of the input state. The problem objective under MRP framing is:
\begin{equation}
\label{eqn_Eg_opt}{
E^u  = \argmax\limits_{\substack{s.t. \sum_{g \in G}|V_g - \overline{V}_G| = 0 \\ \sum_{g \in G}\norm{E^u}_0<B}}\sum_{g \in G}V_g \mbox{.}}
\end{equation}
Here $\overline{V}_G = {\sum_{g \in G}V_g}/{|G|}$. We recast the constrained optimization as unconstrained optimization using augmented Lagrangian \cite{nocedal2006numerical}, as:
%. Thus the unconstrained optimization corresponding to \eqref{eqn_Eg_opt} is:
\begin{align}
    \label{eqn_Eg_uncon_opt}
    \begin{multlined}{
    J  = - \min_{E^u} \sum_{g \in G}{V^g} \\
    - \mu_1(\sum_{g \in G}{V_g - \overline{V_G})}^2 - \mu_2(\min(0,\sum_{g \in G}{\norm{E^u}_0 - B))}^2\\
    -\lambda_1(\sum_{g \in G}{|V_g - \overline{V_G}|)} - \lambda_2(\min(0,\sum_{g \in G}{\norm{E^u}_0 - B)}).}\end{multlined}
\end{align}
Here $\mu_1$ and $\mu_2$ are problem-specific hyperparameters. The Lagrangians of \eqref{eqn_Eg_uncon_opt}, $\lambda_1$ and $\lambda_2$, are updated at every epoch: 
\begin{equation}{
\lambda_1^{new} \gets \lambda_1^{old} + {\mu_1} (\sum_{g \in G}{|V_g - \overline{V}_G|})},
\label{eq:lambda1}
\end{equation}
\begin{equation}{
\lambda_2^{new} \gets \lambda_2^{old} + {\mu_2} (\min(0, \sum_{g \in G}\norm{E^u}_0 - B))}.
\label{eq:lambda2}
\end{equation}

This objective effectively learns the dynamics $P$, of the MRP. Since \eqref{eqn_Eg_uncon_opt} optimizes for Pareto optimality over multiple objectives, the resulting gradient of the stochastic minibatch will tend to be noisy. To prevent such noisy gradient, we train only on either the main objective or one of the constraints at any minibatch. We devise a training schedule where the objective $J$ is optimized without constraint and as it saturates we introduce the equity constraint followed by the edit budget constraint and finally as losses saturate, we force discretizing the edge selection by gradually annealing the temperature $\tau$ of the Gumbel sigmoid. These scheduling schemes are problem-specific and are hyperparameter-tuned for best results. We detail this scheduling strategy in Method \ref{algorithm:MechDesign}. %The effects of the scheduling schemes are visualized in Figure \ref{fig:training_trajectory}.    

\subsubsection{Facility Location as Special Case}
\label{subsub:facility}
An alternative to augmenting edges $E^u$ in a graph $\mathbb{G}$ is to make resources equitably accessible to particles $p_g~\mu(g)$ of different groups $g \in G$ by selecting optimal placement of reward nodes without changing the edges: a facility location problem \cite{krause_leskovec, FishBBFSV2019}. With minor changes to our MRP framework, equitable facility placement can be solved. Specifically, the dynamics of the MRP, $P$ are fixed and the objective is parameterized by the reward vector $R \in \{0,1\}^{|S|}$. The  optimization in \eqref{eqn_Eg_opt} can be adapted to: 
\begin{equation}{
R^u  = \argmax\limits_{\substack{s.t. \sum_{g \in G}|V_g - \overline{V}_G| = 0 \\ \norm{R}_0<B}}\sum_{g \in G}V_g} 
\end{equation}
\begin{equation}
{
R^u(s) = \tfrac{1}{(1+\exp(-(\phi(\Vec{0}) + g_{s})/\tau)}  ,\forall s \in S}
\end{equation}
Besides a small change to the objective, the original MRP framework works for equitable facility placement as well.

{\SetAlgoNoLine%
\begin{Method}[h]
     
    \textbf{Input: }{The original weight matrix $W_g^0$} \;
    \textbf{Initialize} $\tau =1,\lambda_1 =0,\lambda_2 =0,\mu_1,\mu_2, \alpha$\;
    $update\_constraint = True$ \;
     \For{until convergence}{
      \For{Each ADAM optimized minibatch}{
          Sample $s_0\sim\mu(g)$\;
          \eIf{update\_constraint}
          {
            $\theta  \gets \theta - \alpha \nabla_{\theta} (- \mu_1(\sum_{g \in G}{|V_g - \overline{V_G}|)}^2 + \lambda_1(\sum_{g \in G}{|V_g - \overline{V_G}|))}$\;
            $\theta  \gets \theta - \alpha \nabla_{\theta} (- \mu_2(\min(0,\sum_{g \in G}{E_g - B))}^2 - \lambda_2(\min(0,\sum_{g \in G}{E_g - B})) )$\;
          }
          {
            $\theta  \gets \theta - \alpha \nabla (\sum_{g \in G}{V^g}) $\;
          }
          update\_constraint $\gets$ not\ update\_constraint\;
          
          \If{equity schedule condition is met}
          {
            Update $\lambda_1$ (Equation \ref{eq:lambda1})\; % ^{new} \gets \lambda_1^{old} + {\mu_1} (\sum_{g \in G}{V_g - \overline{V_G}})$\;
          }
          \If{edit schedule condition is met}
          {
            Update $\lambda_2$ (Equation \ref{eq:lambda2})\;  %\lambda_2^{new} \gets \lambda_2^{old} + {\mu_2} (\sum_{g \in G}E_g - B)$\;
          }
          \If{temperature schedule condition is met}
          {
            $\tau \gets \tau * \nu$\;
          }
      }
     }
     \caption{Equitable mechanism design in MRP}
     \label{algorithm:MechDesign}
\end{Method}}
\normalsize

%% file: synthetic.tex
\section{Theoretical Analyses}
\phantomsection

\subsection{Computational Complexity of GAEA}

\begin{theorem}
The GAEA problem is in class NP-hard that cannot be approximated within a factor of $(1-\frac{1}{3e})$.
\end{theorem}
\begin{proof}
Consider a subproblem of GAEA: maximization of expected utility of a single group and hence no constraints on equity. Let us assume there is only one reward node $r \in V$ and drop the constraint of the node being reachable in $T$ steps. Now the problem reduces to augmenting a set of edges,  $E^u=\{E^u(i,j)\not\in E\}$, to improve closeness $c_r = 1/d_{vr}$ of reward nodes $r$. Here $d_{vr}$ is distance to vertex $v$ from reward node $r$; the optimization problem reduces to
\begin{equation}{
    E^u  = \argmax\limits_{\substack{s.t. \norm{E^u}_0<B }}{c_r}}\mbox{.}
\end{equation}
The GAEA problem now reduces to the Maximum Closeness Improvement Problem \cite{mci} which through Maximum Set Cover, has been proven to be NP-hard and cannot be approximated within a factor of $(1-\frac{1}{3e})$, unless P = NP \cite{feige1998threshold}.
\end{proof}

\subsection{Computational Complexity of Facility Placement} Facility placement is proven to be submodular \cite{krause_leskovec}. For the unit-cost case there exists a greedy solution that is $(1-\frac{1}{e})$-approximate. There is a tighter problem-dependent $\frac{1}{2}(1-\frac{1}{e})$ bound \cite{krause_leskovec}.

\subsection{Complexity Analysis}
\phantomsection

\subsubsection{EMD-MRP}
Here we analyze the time complexity of each minibatch while training in our MRP. The complexity of the forward pass is mainly in estimating the expected value function of each group $v_g(s_0) = \sum_{t = 0}^{T-1} \gamma^tRP^ts_0$, which is dominated by the computation of $P^ts$. This can be done by recursively computing $Ps$ for $T$ time steps, resulting in minibatch complexity $O(B \cdot |G| \cdot T \cdot |V|^2)$. Alternatively $P^t$ can be computed once every minibatch, with complexity $O(|V|^{2.37})$ \cite{matrixMul}, resulting in overall complexity of $O(T \cdot |V|^{2.37} + \mathbb{B} \cdot |G| \cdot |V|^2)$. For large networks $|V| \gg BGT$, and the complexity reduces to $O(|V|^2)$.

\subsubsection{GECI}
Introducing group equity into the work of \citet{mci}, the time complexity of greedy improvement strategy becomes $O(B \cdot |E| \cdot |V| \cdot |G| \cdot O( U_g ))$, where $O( U_g )$ is the complexity of computing the utility of a group.

\subsection{On Mixing Time}
Here we analyze the mixing time of EDM-MRP formulation. Let us add a virtual absorption node $r_a$ to the graph $\mathbb{G}$, such that all reward nodes $r \in R$ almost surely transition to $r_a$.
The state distribution at time step $t$ is given by $s_t = P^ts_0$. At optimality, in a connected graph, the objective is to have all nodes reach a reward node within timestep $T$ and therefore reach absorption node $r_a$ within $T+1$ timesteps, which results in a steady-state distribution $\lim_{t \to \infty} s_t = r_a$.
The convergence speed of $s_0$ to $r_a$ is given by the asymptotic convergence factor \cite{xiao2004fast}:
\begin{equation}
    {
    \rho(P) = \underset{s_0 \neq r_a}{\sup} \underset{t \to \infty}{\lim} {\Big( \frac{{||s_{t}-r_a||}_2}{{||s_{0}-r_a||}_2}\Big)}^{1/t} , s_0\sim\mu(g) \forall g \in G}
\end{equation}
and associated convergence time
\begin{equation}{
T + 1 =\tfrac{1}{\log(1/\rho)}} \mbox{.}   
\end{equation}
While discount factor $\gamma$ relates to expected episode length $T$ \cite{discountfactorepisodelen} by
\begin{equation}{
T =\tfrac{1}{1-\gamma}} \mbox{.}  
\label{T_gamma}
\end{equation}
Let $\rho_{g_{adv}}$ and $\rho_{g_{dis-adv}}$ be the convergence factors corresponding to the most advantageous and disadvantageous group that require least $T_{g_{adv}}$ and most $T_{g_{dis-adv}}$ timesteps to $r_a$. The constraints of the optimization is to make $T_{g_{dis-adv}}$ match $T_{g_{adv}}$ with discrete edits to graph $\mathbb{G}$ less than budget $B$. For such matching to happen, from \eqref{T_gamma} the choice of $\gamma$ should be.
\begin{equation}{
\gamma \leq 1 - \tfrac{1}{T_{g_{adv}}}} \mbox{.}  
\label{opt_gamma}
\end{equation}
Hence the optimal dynamics $P^{*}$ is bounded by choice of the discount factor $\gamma$, in \eqref{eqn_Vg} of the MRP which in turn is influenced by the given graph's $\mathbb{G}$ inequity of convergence factors $\rho(P_{dis-adv})$ - $\rho(P_{adv})$ and budget $B$.

\section{Evaluation: Synthetic Graphs}

Many real-world sociotechnical networks can be approximated by synthetic graphs; hence we evaluate our method on several synthetic graph models that yield instances of graphs with a desired set of properties. We evaluate our proposed graph editing method with respect to the parameters of the graph model.
%We consider the following. 
\phantomsection

\subsection{Erd\"{o}s-R\'{e}nyi Random Graph (\er{})}

The Erd\"{o}s-R\'{e}nyi random graph is parameterized by $p$, the uniform probability of an edge between two nodes. The expected node degree is therefore $p|N|$, where $|N|$ is the number of nodes in $G$. We use this model to measure the effectiveness of our method with varying graph densities. As the density increases, it will be more difficult to affect the reward of nodes through uncoordinated edge changes. 

\subsection{Preferential Attachment Cluster Graph (\texttt{PA})}

The Preferential Attachment Cluster Graph graph model \cite{holme2002growing} is an extension of the \baname{} graph model. This model is parameterized by $m$ added edges per new node, and the probability $p$ of adding an edge to close a triangle between three nodes. %The BA model iteratively adds nodes to a graph by connecting each new node with m edges, proportional to the degree of existing nodes. This yields a power law degree distribution with probability of nodes with degree $k$: $P(k) \sim k^{-3}$. The cluster graph PA model generalizes to the base BA model at $p=0$.We use this model to evaluate performance on graphs with varying clustering. This is similar to the \er{} setting, where higher clustering makes it more difficult to traverse farther in the graph, except under the same edge density.

\subsection{\clname{} Power Law Graph (CL)}

The \citet{chung2002connected} model yields a graph with expected degree distribution of an input degree sequence $d$. We sample a power-law degree distribution, yielding a model parameterized by $\gamma$ for $P(k) \sim k^{-\gamma}$. This is the likelihood of sampling a node of degree $k$. In this model, $\gamma=0$ yields a random-degree graph and increasing $\gamma$ yields more skewed distribution (fewer high-degree nodes and more low-degree nodes). We use this graph model to measure performance with respect to node centrality. As $\gamma$ increases, routing is more likely through high-degree nodes (their centrality increases). We place rewards at high-degree nodes. %We anticipate expected rewards increases with $\gamma$ on uniform edge weights.

\subsection{Stochastic Block Model (SBM)}

The Stochastic Block Model \cite{holland1983stochastic} samples edges within and between $M$ clusters. The model is parameterized by an $[M \times M]$ edge probability matrix. Typically, intra-block edges have a higher probability: $m_{i,i} > m_{i,j}, \textrm{where } j \neq i$.We use this model to measure performance at routing between clusters. In this setting, we instantiate two equal-sized clusters with respective intra- and inter-cluster probability: $[0.1, 0.01]$. We sample particles starting within each cluster. This experiment measures ability to direct particles into a sparsely connected area of the graph. This may be relevant in social or information graphs where rewards are only available in certain communities. % and our method proposes interventions to 

\subsection{Edge and Particle Definitions}

For each graph ensemble, we create a graph edge set, which we then sample two or more edge-weight sets and sets of diffusion particles. For simplicity we will cover sampling two, for  \textbf{\textcolor{rossocorsa}{red}} and a \textbf{black} diffusion particles. For all synthetic experiments, for black diffusion particles we define edge weights proportional to node degree:
\begin{equation}{
w_{i,j} = \mathrm{deg}(i)\cdot\mathrm{deg}(j) \mbox{.}}
\end{equation}
For red particles, we define edge weights inversely proportional to degree nodes analogous to the above.   
% \begin{equation}
% w^r_{i,j} = \frac{1}{\mathrm{deg}(i)\cdot\mathrm{deg}(j)}  
% \end{equation}
For each diffusion step, a particle at node $i$ transitions to a neighboring node by sampling from the normalized distribution weight of edge incident to $i$. This  weighting means black particles probabilistically favor diffusion through high-degree nodes, whereas red particles favor diffusion through low-degree nodes. We use \textit{random} initial placement of particles within the graph. The difference in edge diffusion dynamics thus constitute bias within the environment.

\subsection{Problem Instances: Reward Placement}
\label{subsec:reward_placement}

For each synthetic graph ensemble, we specify two different problems by varying the definition of reward nodes on the graph. For the \textbf{\hd{}} problem, we sample $k=3$ nodes proportional to their degree: 
\begin{equation}{
P(i) = \tfrac{\mathrm{deg}(i)}{\sum_{j \in V}{\mathrm{deg}(j)}} \mbox{.}} 
\end{equation}
For the \textbf{\ld{}} problem, we sample $k = 3$ nodes inversely proportional to their degree, analogous to the above.
% \begin{equation}
% P(i) = \frac{\mathrm{deg}(i)^{-1}}{\sum_{j \in V}{\mathrm{deg}(j)^{-1}}} \mbox{.}
% \end{equation}
This means that especially in power-law graphs such as \pa{} and \cl{}, black particles which favor high-degree nodes are advantaged and should have a higher expected reward. However, we also hypothesize that black particles could be advantaged in the \ld{} placement, because routing necessarily occurs through high-degree nodes for graphs with highly skewed degree distributions. Overall, we hypothesize the \ld{} problem instance is relatively harder for graph augmentation methods.

% \begin{figure}[t]
%   \centering
%   \includegraphics[width=\columnwidth]{example-image-b}
%   \caption{synthetic 1: CL experiment}
%   \label{fig:syn1}
% \end{figure}

% \begin{figure}[t]
%   \centering
%   \includegraphics[width=\columnwidth]{example-image-c}
%   \caption{synthetic 2: ER experiment}
%   \label{fig:syn2}
% \end{figure}

% \begin{figure}[t]
%   \centering
%   \includegraphics[width=\columnwidth]{example-image-c}
%   \caption{Interpretabiliy 1: edits on a graph (sbm)}
%   \label{fig:edits_on_graph}
% \end{figure}

\begin{figure*}[ht]
\centering
\begin{subfigure}{\threesubfig\textwidth}
  \includegraphics[width=\figwidth\columnwidth]{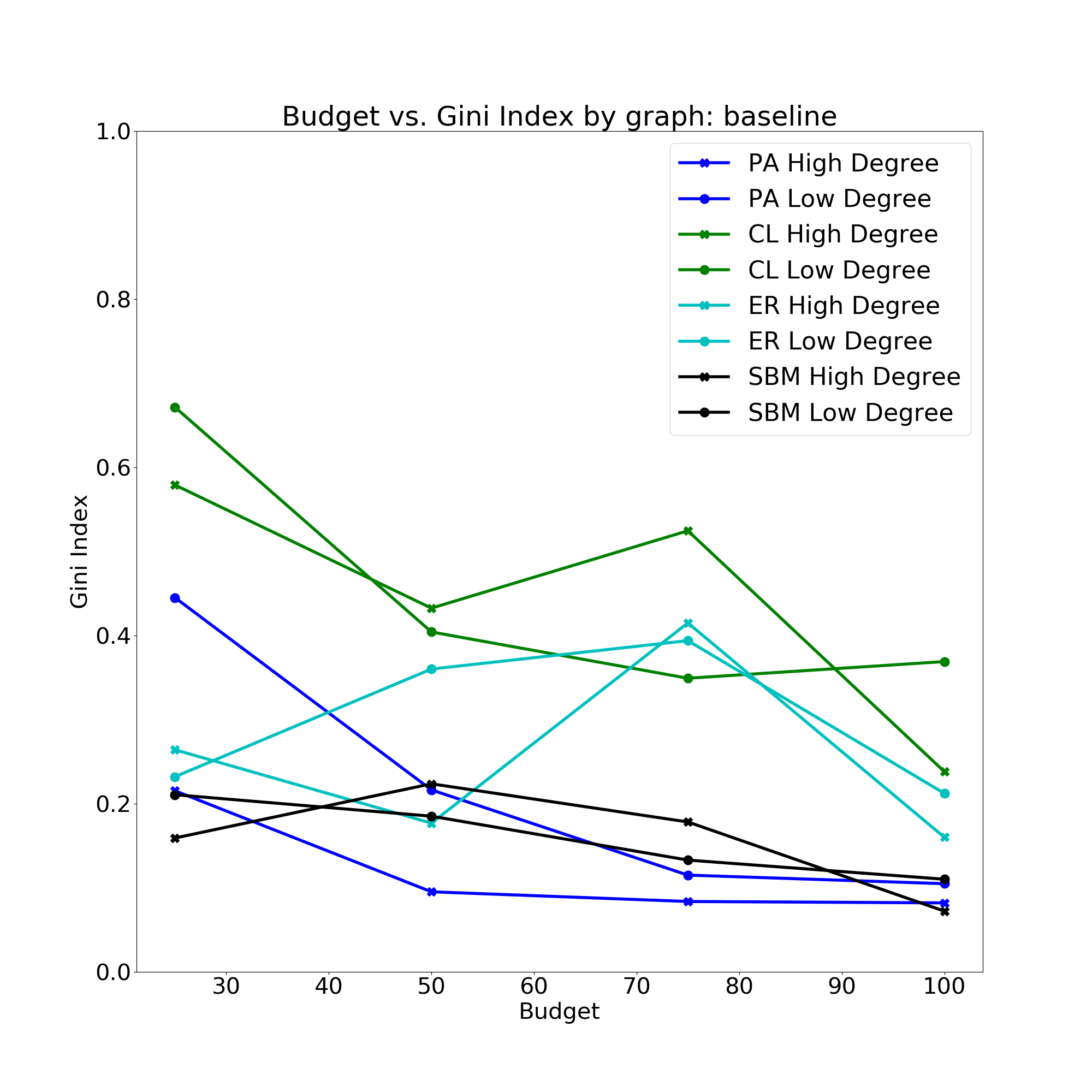}
  \caption{GECI baseline} 
  \label{subfig:budget_v_gini_1}
  \end{subfigure}
\begin{subfigure}{\threesubfig\textwidth}
  \includegraphics[width=\figwidth\columnwidth]{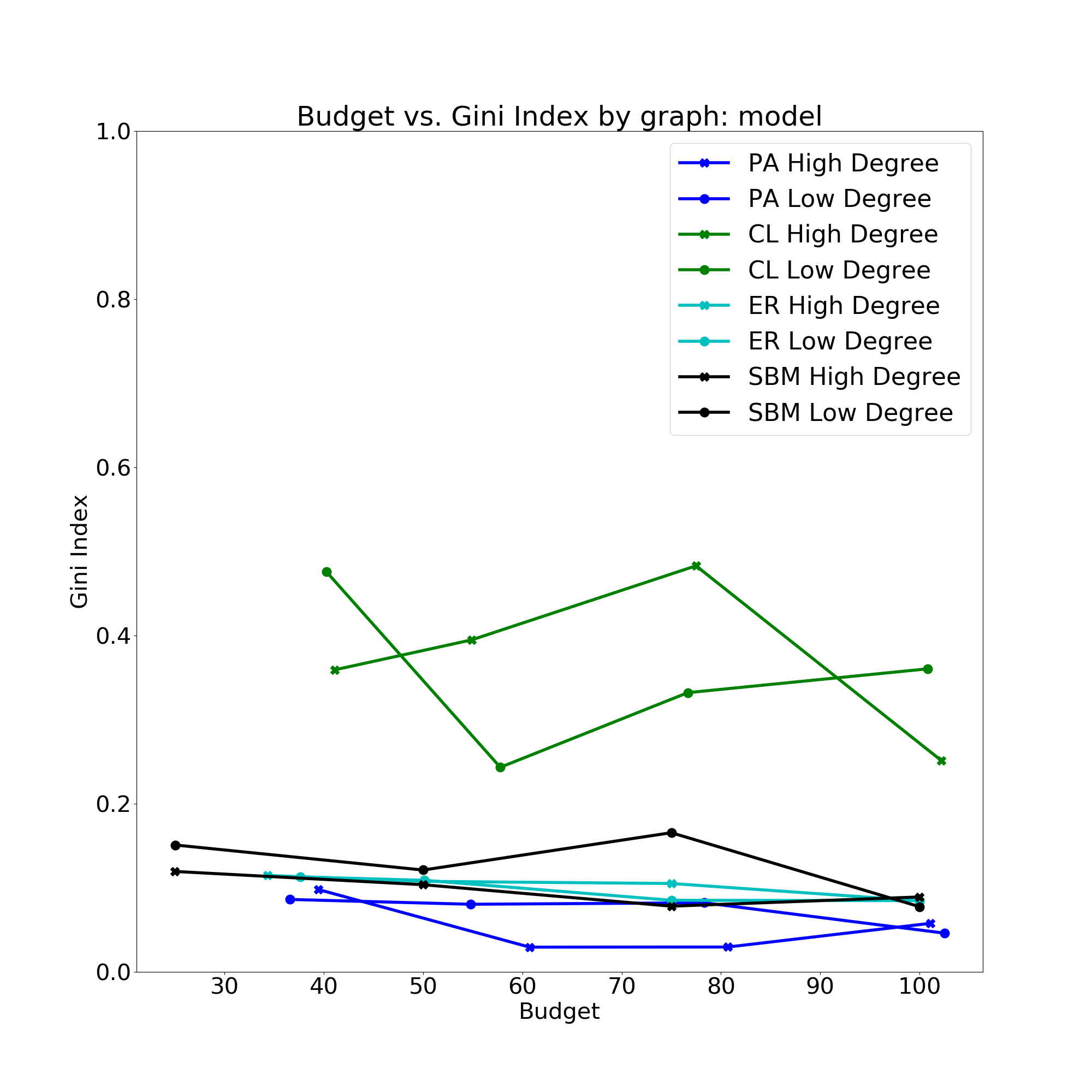}
  \caption{Proposed method} 
  \label{subfig:budget_v_gini_2}
  \end{subfigure}
 \begin{subfigure}{\threesubfig\textwidth}
  \includegraphics[width=\figwidth\columnwidth]{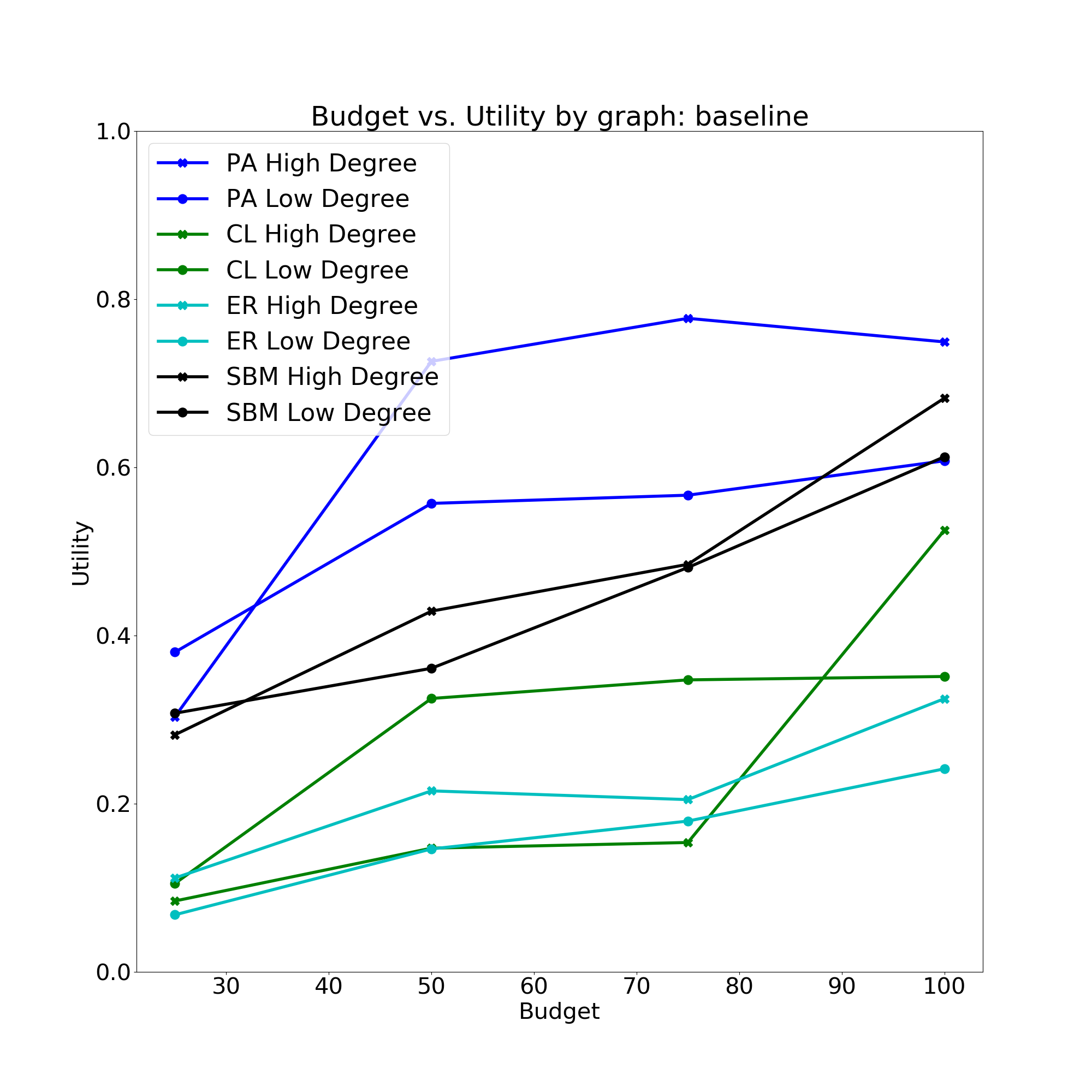}
  \caption{GECI baseline} 
  \label{subfig:budget_v_utility_1}
  \end{subfigure}
\begin{subfigure}{\threesubfig\textwidth}
  \includegraphics[width=\figwidth\columnwidth]{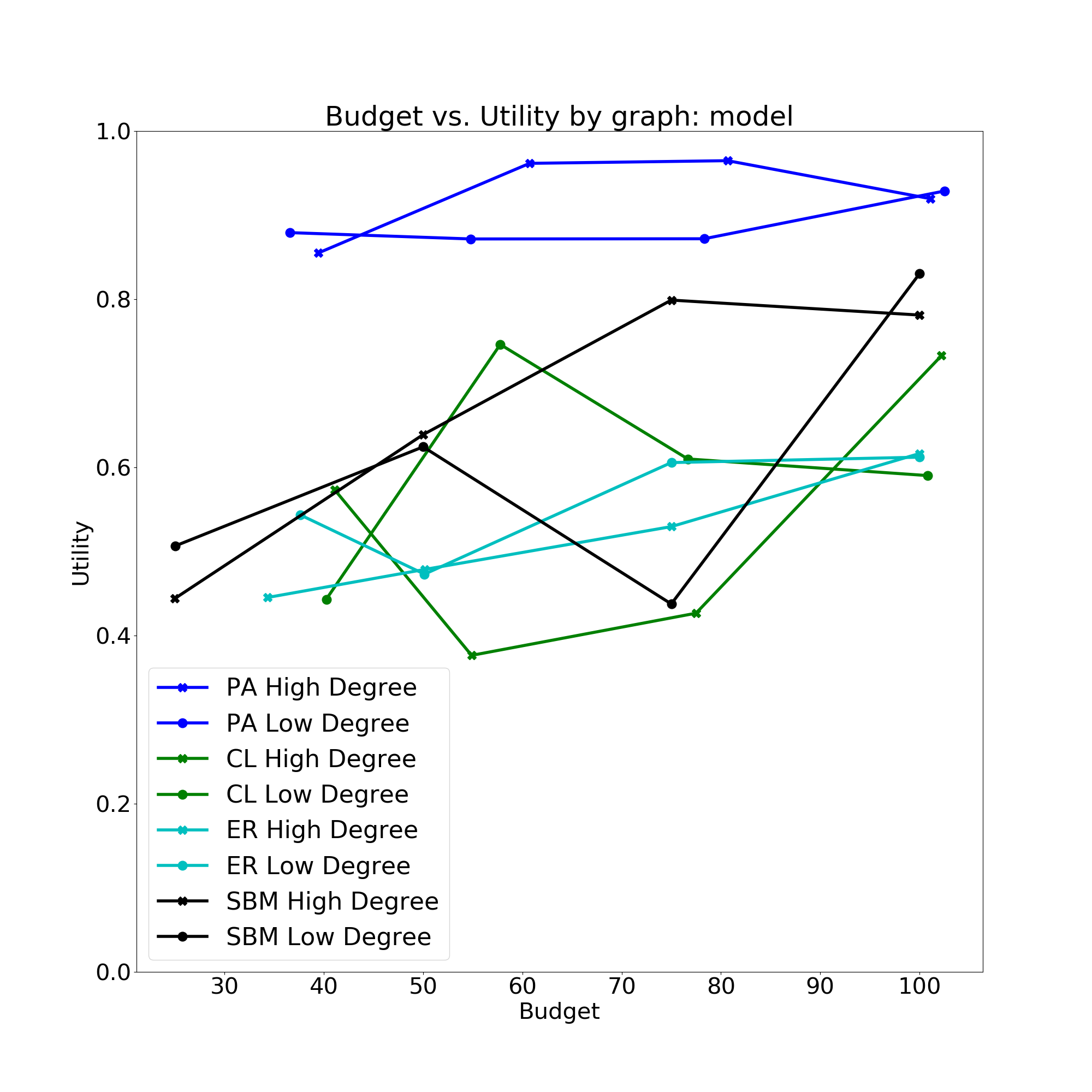}
  \caption{Proposed model} 
  \label{subfig:budget_v_utility_2}
  \end{subfigure}
\caption{(a-b) Budget vs Gini Index. This shows the Gini Index for varying budgets for (a)  the GECI baseline and (b) our proposed method. Our proposed model performs better, particularly at a smaller budget. (c-d) Budget vs Utility. This shows the utility for varying budgets for (c)  the GECI baseline and (d) our proposed method. Our proposed method outperforms the deterministic baseline on all graphs at all budgets.}
\label{fig:budget_v_gini}
\end{figure*}

%% file: results.tex
\subsection{Evaluation}
We evaluate equity and utility for the graph produced by our method, comparing against the input graph and the baseline editing method. To define utility, we estimate the expected reward per group by repeated Monte Carlo sampling of weighted walks through the graph. First, we sample the starting node of an individual with respect to their initial distribution, then estimate their expected reward over weighted walks from the starting node. Repeatedly sampling individuals yields an expected utility for the graph with respect to each group. We measure the total expected reward per population (utility), and the difference in expected reward between classes (equity). Further, while our model only optimizes the expectation, it performs surprisingly well at minimizing the Gini Index. 

\subsection{Evaluation Metrics}
\phantomsection

\subsubsection{Average Reward}
We measure three graphs in experiments: the initial graph before editing, and outputs of the GECI baseline and our proposed method. We simulate 5000 weighted walks by the initial distributions of each particle type. Average reward is aggregated across these particle types.

\subsubsection{Gini Index}
The Gini Index  is a measure of inequality. It measures the cumulative proportion of a population vs. the cumulative share of value (e.g. reward) for the population. At equality, the cumulative fraction of the population is equal to the cumulative reward. The measure is the deviation from this $x=y$ line, with $1$ being total inequality, and $0$ total equality.

\subsection{Synthetic Results}
Figure \ref{fig:budget_v_gini} overviews our synthetic results. We see that on all graphs and over almost all budgets, the proposed model outperforms the baseline. Further, we especially see that in the low-budget scenario, our model outperforms on Gini Index. Our model outperforms the baseline as much as $0.5$ in utility under the same budget. In particular, PA and ER graphs improve the most. Figure \ref{fig:scatter_full} gives the main empirical result of our synthetic experiments. On the $8$ experiments of $4$ different synthetic graphs, we plot the utility vs. the Gini Index across Monte Carlo simulations of the original, baseline, and proposed method graphs. The bottom-right reflects the best performance. Notice the proposed model outperforms the baseline on all synthetic graphs. Particularly notable, the Chung-Lu Power Law graph is the worst performing original graph in terms of both utility and Gini Index. However, on the low-degree problem instance, our method nearly doubles the utility performance of the baseline.

\begin{figure}
\centering
  \begin{subfigure}{\threesubfig\textwidth}
  \centering
      \includegraphics[width=\columnwidth]{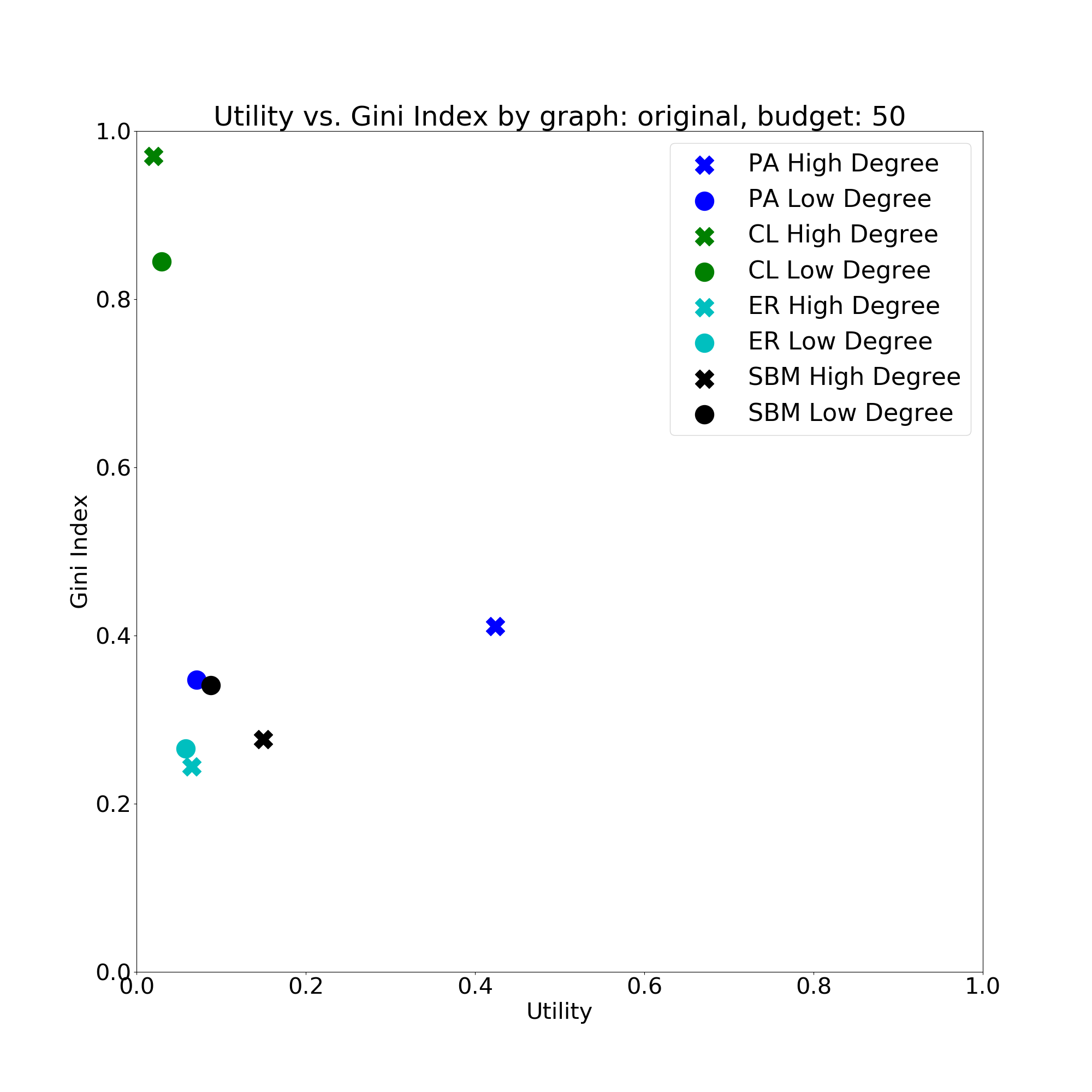}
  \caption{Unedited Graph}
  \label{subfig:scatter1}
  \end{subfigure}
    \begin{subfigure}{\threesubfig\textwidth}
  \centering
  \includegraphics[width=\columnwidth]{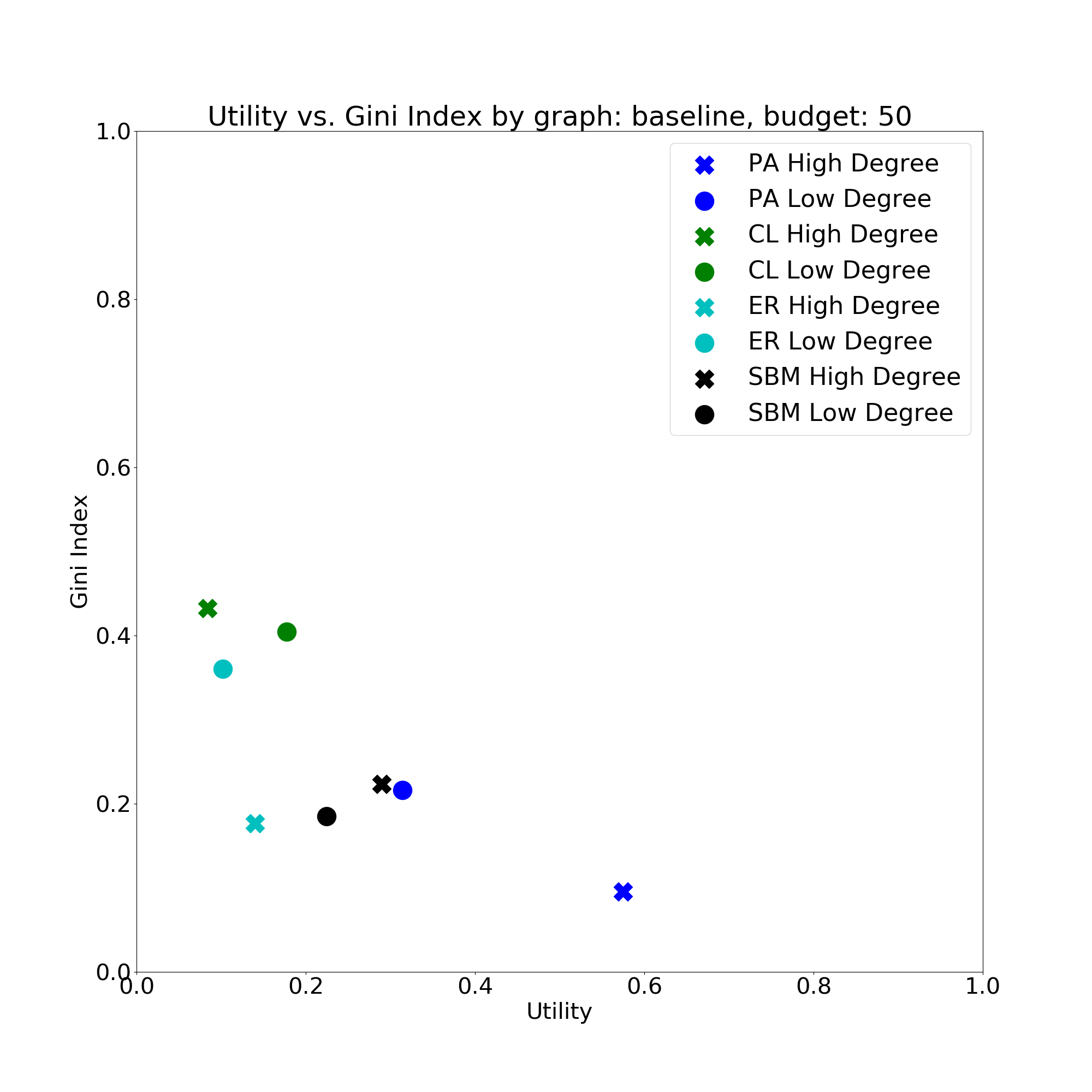}
  \caption{GECI Baseline}
  \label{subfig:scatter2}
  \end{subfigure}
   \begin{subfigure}{\threesubfig\textwidth}
    \centering
        \includegraphics[width=\columnwidth]{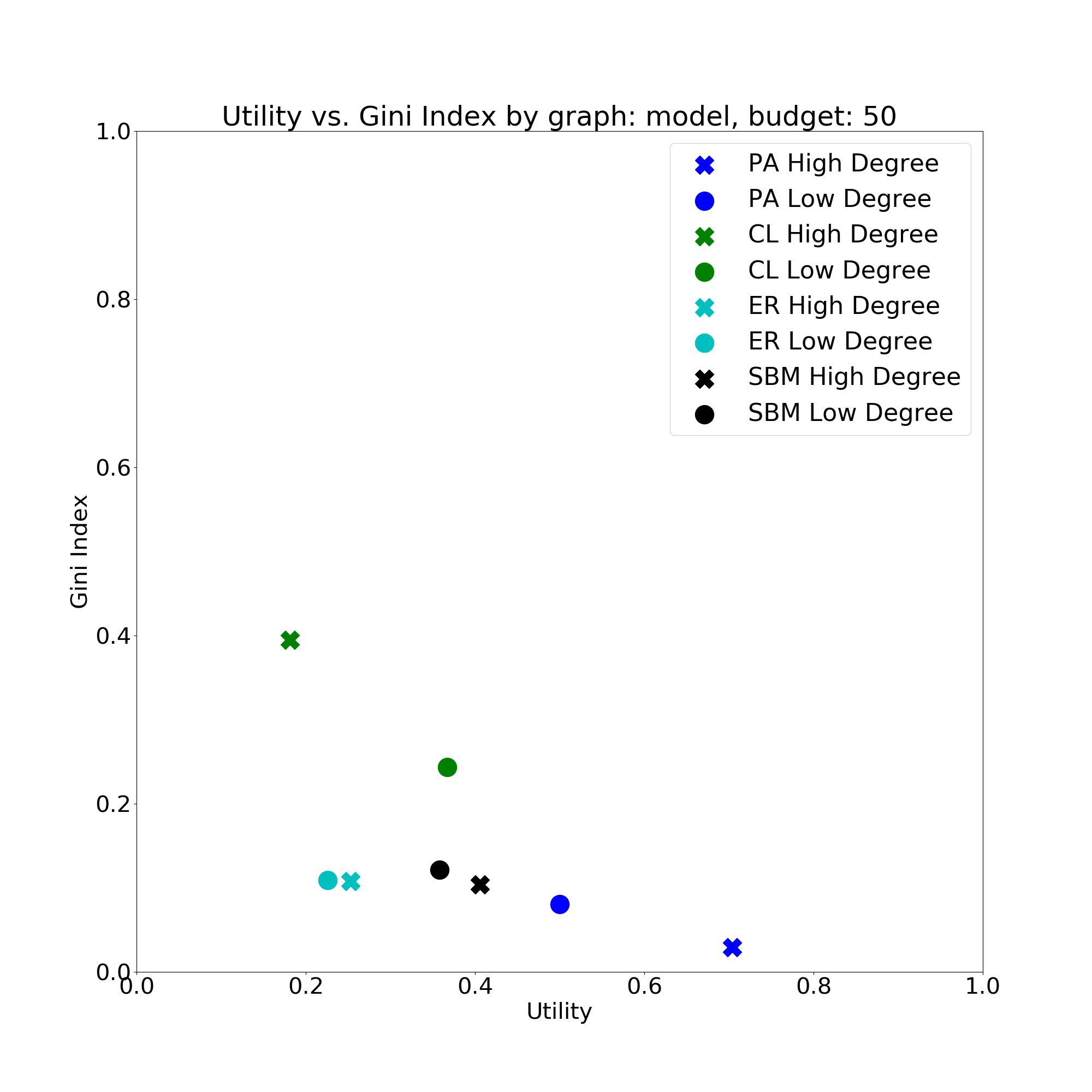}
  \caption{Proposed method}
  \label{subfig:scatter3}
  \end{subfigure} 
  \caption{A comparison of Utility vs.\ Gini Index across 8 synthetic experiments. Bottom-right is best. The proposed method performs best on all experiments.}
  \label{fig:scatter_full}
 \end{figure}

\subsubsection{Facility Placement}
As discussed, our proposed model also solves the facility placement problem. The problem selects a number of nodes which maximizes the reward for particles sampled onto the graph from initial distributions. Figure \ref{fig:facility_placement} shows a simple experiment adapting our model for this problem. For brevity, we only show a small example experiment. In black, we see the curve of the Gini Index decreasing on increased budget of $15$ for a synthetic PA graph of size $|N| = 200$. At the same time, the average utility increases to approximately the same budget. Note the initial location of PA high-degree under budget $3$ using the greedy PA high-degree heuristic (Figure \ref{subfig:scatter1}). This is approximately 4 for both Gini Index and Reward. Our method maintains far lower Gini Index.
Therefore this node set largely covers the transition dynamics of the initial distributions. 
\begin{figure}
      \centering
      \includegraphics[width=\columnwidth]{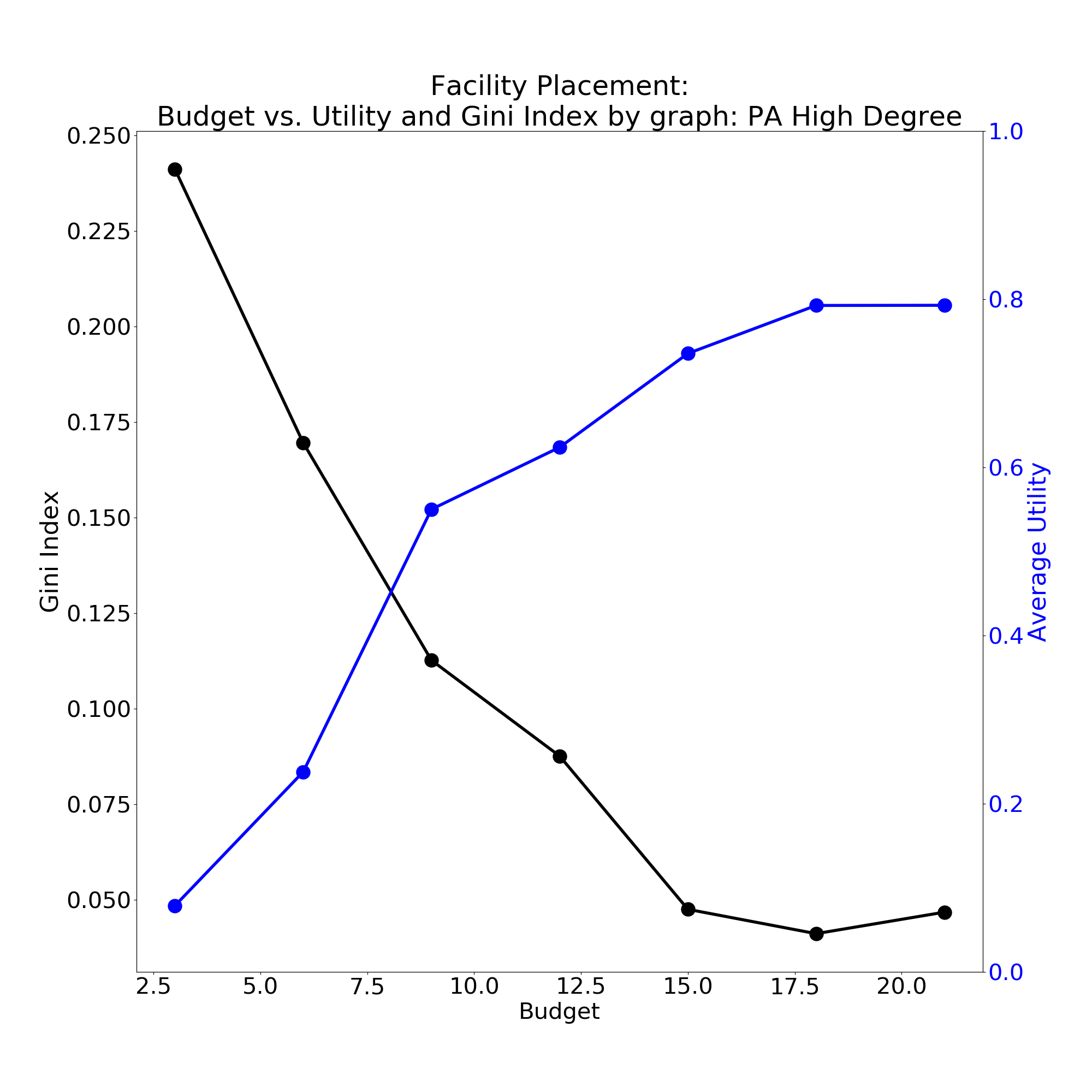}
      \caption{Facility placement results showing varying budget (facilities) vs. total Gini index and utility.}
    \label{fig:facility_placement}
\end{figure}
\input{results_real}

%% file: results_real.tex
\begin{table}
\centering
\begin{tabular}{c||ccc}
\multicolumn{4}{c}{\textbf{Chicago Schools}} \\
\hline
              & Initial  & GECI  & Proposed Model \\
\hline\hline
Avg. Utility  & 0.20     & 0.21      & 0.90  \\
Gini Index    & 0.62     & 0.65      & 0.07 
\end{tabular}
\caption{Chicago Public School with budget 400}
\label{tab:chicago}
\end{table}

\section{Real-World Applications}
\phantomsection

\subsection{Equitable School Access in Chicago}
\label{subsec:chicago}

In this section we study school inequity in the city of Chicago. We infer a coarse transportation network using the trajectory data of public bus lines from the Chicago Transit Authority (CTA).\footnote{https://www.transitchicago.com/data/} Nodes are given by route intersections, and edges are inferred from neighboring preceding and following intersections.  This yields a graph with $2011$ nodes and $7984$ edges. We collect school location and quality evaluation data from the Chicago Public School (CPS) data portal.\footnote{https://cps.edu/SchoolData/} We use the $2018$-$2019$ School Quality Rating Policy assessment and select elementary or high schools with an assessment of "Level 1+," corresponding to ``exceptional performance'' of schools over the $90$th percentile. We select only non-charter, ``network'' schools which represent typical public schools. We use geolocation provided by CPS to create nodes within the graph. We attach these nodes to the graph using $2$-nearest neighbor search to the transportation nodes. Finally, we collect tract-level demographic data from the $2010$ census.\footnote{https://factfinder.census.gov/} We sample three classes of particle onto the network, representing White, Black, and Hispanic individuals by their respective empirical distribution over census tracts. We then randomly sample nodes within that tract to assign the particle's initial position. We equally set initial edge weights for all groups, with weights inversely proportional to edge distance. 

Table \ref{tab:chicago} shows the result for a budget of $400$ edges in the Chicago transportation network. We see that the baseline is surprisingly ineffective at increasing reward. Our method successfully optimizes for both utility and equity,  achieving a very high performance on both metrics. Note that both the baseline and our proposed model make the \textit{same number of edits} on the graph. We hypothesize the baseline performs poorly on graphs with a high diameter such as infrastructure graphs. Recall we similarly saw the baseline performs poorly on ER (Figure \ref{subfig:scatter2}), which has relatively dense routing. In contrast, our model learns the full reward function over the topology and can discover edits at the edge of its horizon.

\subsection{Equitable Access in University Social Networks}

To demonstrate the versatility of the proposed methods, we also apply them to reducing inequity in social networks. Social networks within universities and organizations may enable certain groups to more easily access people with valuable information or influence. We report experiments for university social networks using the Facebook100 data \cite{TraudMP2012}.
\begin{figure}
    \centering
    \begin{subfigure}{\bannerfig\columnwidth}
        \centering
        \includegraphics[width=\columnwidth]{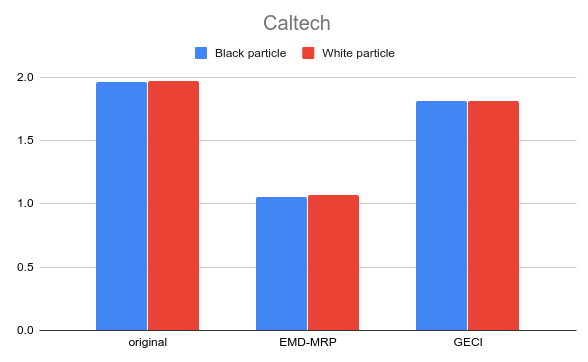}
        \caption{Caltech.}
       \label{MSP_Clatech}
    \end{subfigure} 
    \begin{subfigure}{\bannerfig\columnwidth}
        \centering
        \includegraphics[width=\columnwidth]{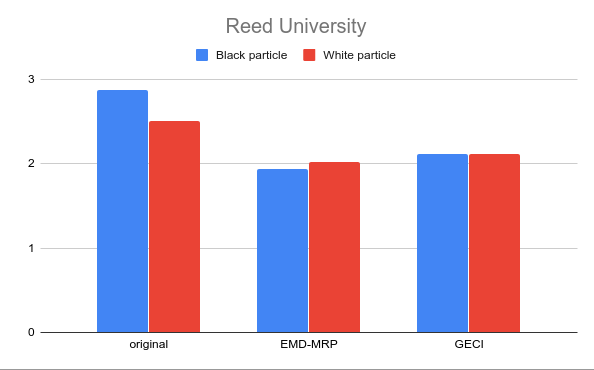}
        \caption{Reed College.}
        \label{MSP_Reed}
    \end{subfigure}
    \begin{subfigure}{\bannerfig\columnwidth}
        \centering
        \includegraphics[width=\columnwidth]{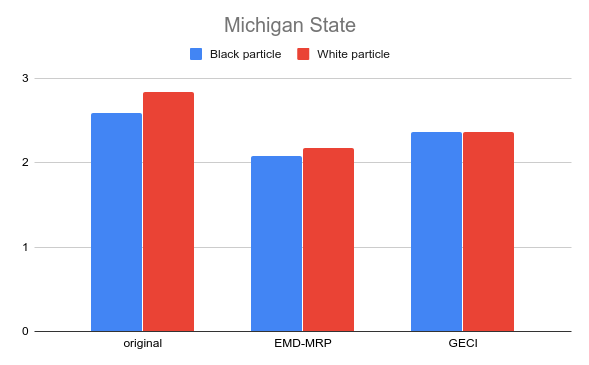}
        \caption{Michigan State.}
        \label{MSP_Michigan}
    \end{subfigure}
    \caption{Mean shortest path of gender groups from the influence nodes}
\end{figure}   
The Facebook100 dataset contains friendship networks at 100 US universities at some time in 2005. The node attributes of this network include: dorm, gender, graduation year, and academic major. Analyzing Facebook networks of universities yield sets of new social connections that would increase equitable access to certain attributed nodes across gender groups. We define popular seniors as the reward nodes and the objective is for freshmen of both genders to have equitable access to these influential nodes. In this experiment we mask the specific gender information by the term white and black particles. We demonstrate our method on three universities. 

The results are in Figures \ref{MSP_Clatech}, \ref{MSP_Reed}, and \ref{MSP_Michigan} which show the mean shortest path of gender groups from the influence nodes at Caltech, Reed College, and Michigan State University, respectively. Table~\ref{tab:gini1} shows the intra-group Gini index. With sufficient hyperparameter tuning, our RL method within our novel MDP framework consistently outperforms the greedy GECI baseline on intra-group Gini index and minimizing overall shortest path of the freshman from the influence node across groups. On the other hand the difference between average shortest distance between groups, GECI maintains tighter margin when compared to our EMD-MRP approach. This is explained in the slackness in the constrained optimization of the EMD-MRP approach.

\begin{table}[h]
\small
\begin{tabular}{c||ccc}
\multicolumn{4}{c}{} \\
\hline
              & Caltech  & Reed & Mich.\ State \\
\hline\hline
num.\ nodes, $|V| $  &  770    &   963    & 163806  \\
num.\ edges, $|E|$  &  33312    &  37624     &  163806 \\
num.\ editable edges $|A|$  &   336597   &  474439     & 3958660  \\
\end{tabular}
\caption{Graph properties of university social networks}
\label{tab:gini1}
\end{table}

\begin{table}[h]
\small
\begin{tabular}{c||ccc}
\multicolumn{4}{c}{} \\
\hline
              &  Original & EMD-MRP & GECI \\
\hline\hline
Reed  &  0.214    &   0.093    & 0.153  \\
Caltech  &  0.092    &  0.065     &  0.812 \\
Michigan State  &   0.115   &  0.086     & 0.157  \\
\end{tabular}
\caption{Intra-group Gini Index}
\label{tab:gini2}
\end{table}

\section{Conclusion}

In this work, we proposed the \textit{Graph Augmentation for Equitable Access} problem, which entails editing of graph edges, to achieve equitable diffusion dynamics-based utility across disparate groups. We motivated this problem through the application of equitable access in graphs, and in particular applications, equitable access in infrastructure and social networks. We evaluated our method on extensive synthetic experiments on 4 different synthetic graph models and 8 total experimental settings. We also evaluated on real-world settings.

There are many avenues for future work.
First, our reward function is somewhat limiting. Ideally, individuals could collect rewards on the graph in a number of ways. Second, we measured a particular \textit{equal opportunity} fairness which has a practical mapping to our application setting. However, other definitions of group or subgroup-level fairness may not be easily translatable to a graph routing/exploration problem.

%% file: appendix.tex
\section{Appendix}
In this appendix, we provide some further details on the experiments in the main text.
\phantomsection

\subsection{Training Trajectories}
Fig.~\ref{fig:training_trajectory} are the train trajectories on synthetic graphs. They indicate the trends of mean utility across, difference in utility across group and the budget constraint. The kinks in these learning curve are due to different scheduling schemes kicking in. Refer to Table \ref{tab:hyperparam} for these scheduling details for synthetic graph

\begin{figure*}
\centering
  \begin{subfigure}{.28\textwidth}
  \centering
  \includegraphics[width=\columnwidth]{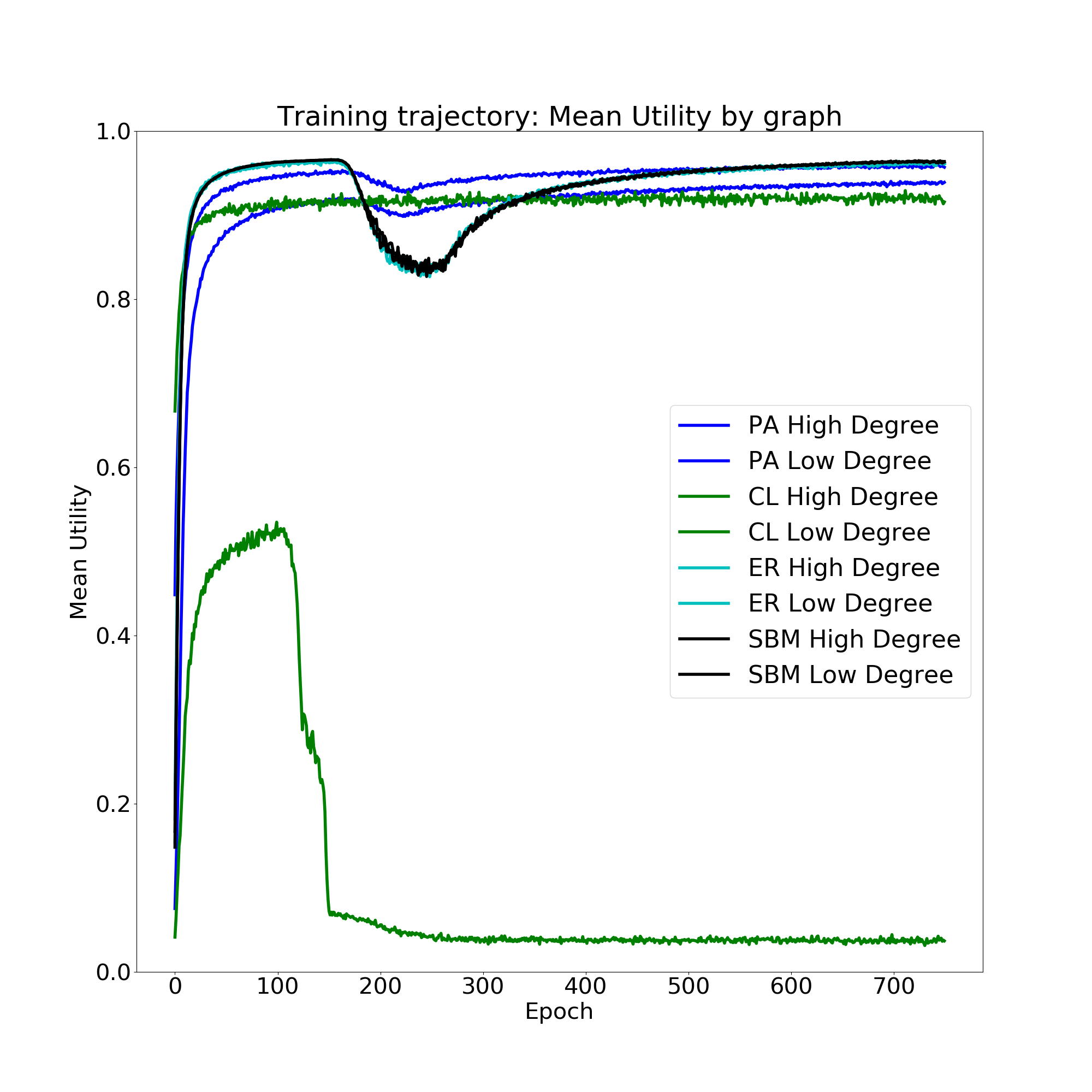}
  \caption{Mean Utility}
  \label{subfig:trajectory1}
  \end{subfigure}
    \begin{subfigure}{.28\textwidth}
  \centering
  \includegraphics[width=\columnwidth]{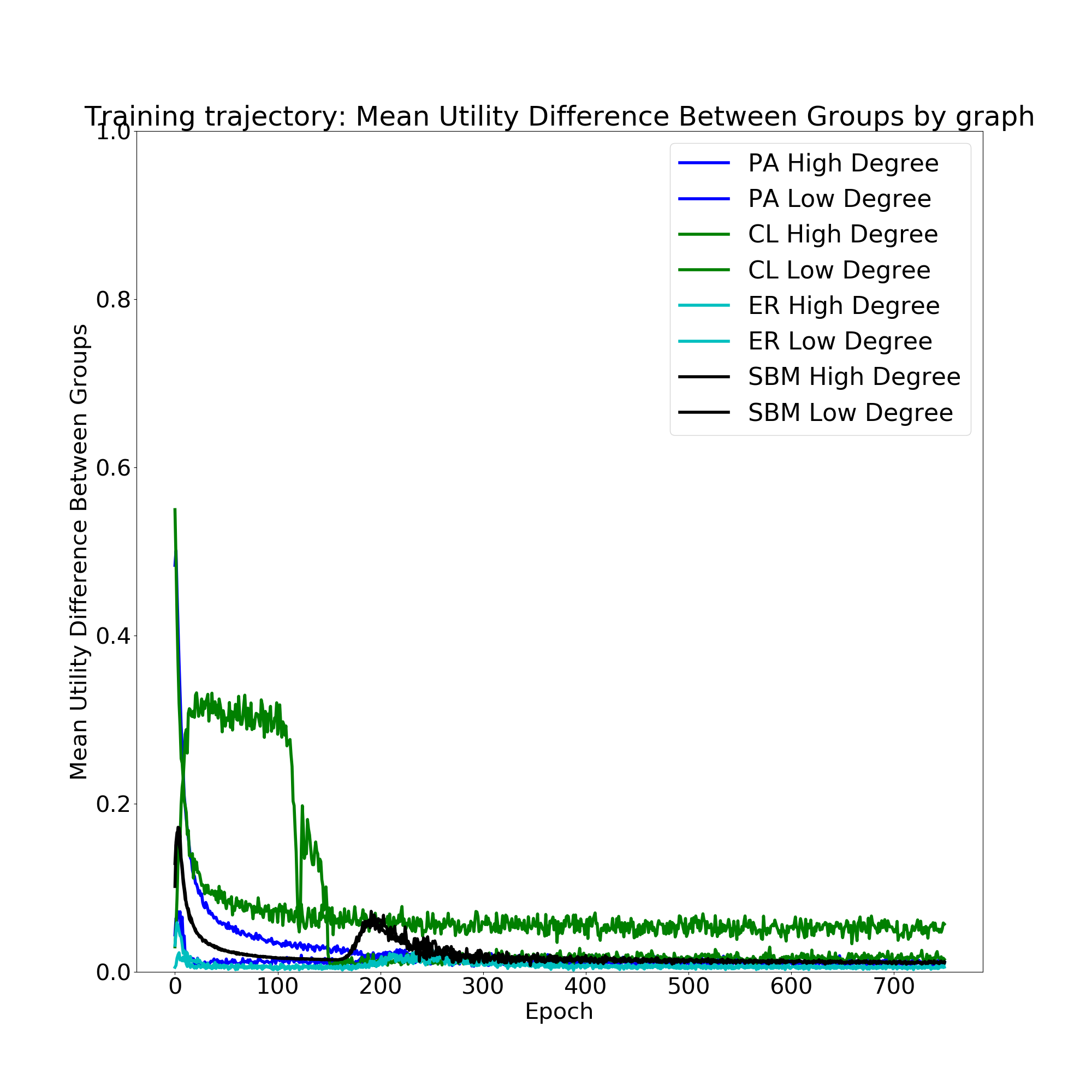}
  \caption{Mean Utility Difference between groups}
  \label{subfig:trajectory2}
  \end{subfigure}
   \begin{subfigure}{.28\textwidth}
    \centering
        \includegraphics[width=\columnwidth]{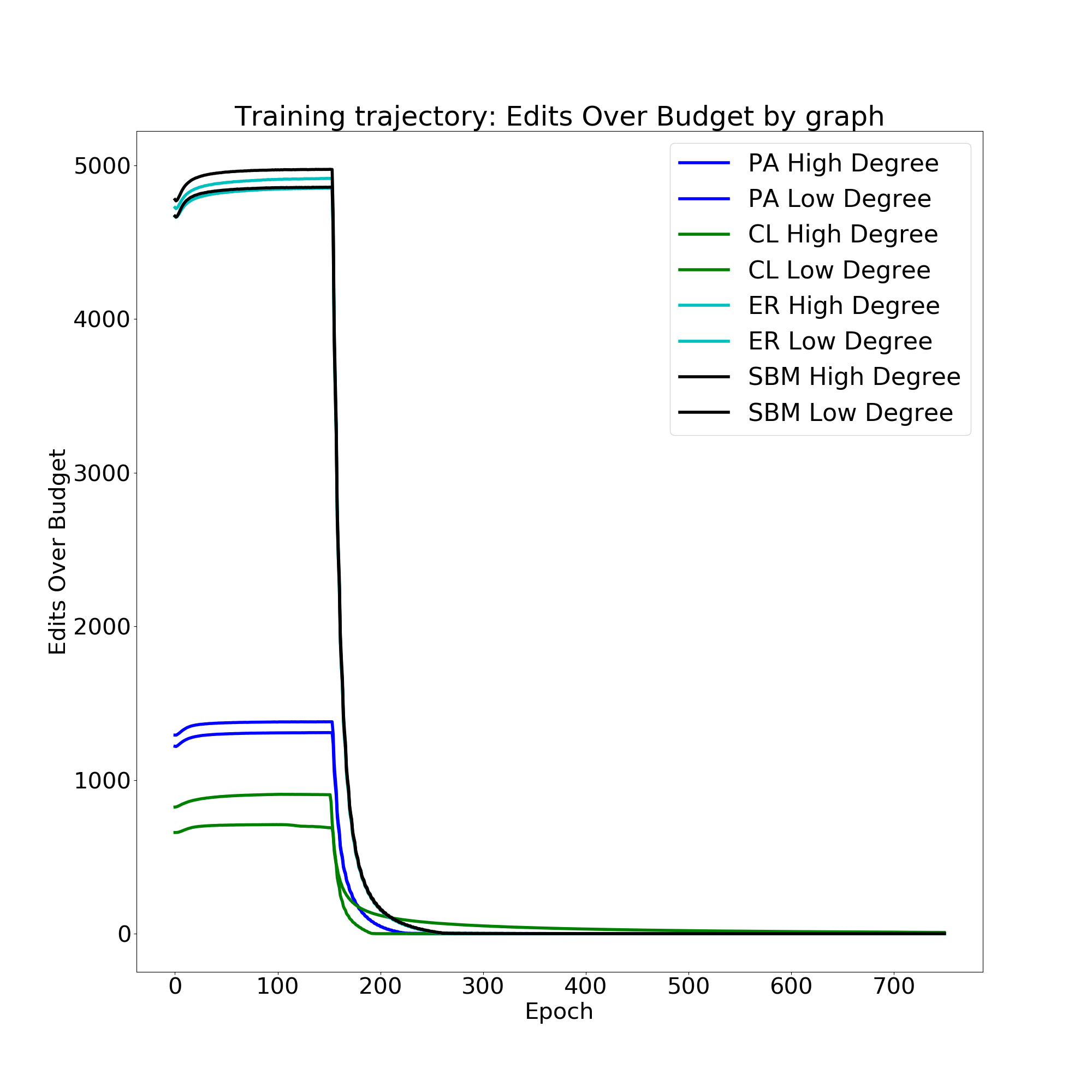}
  \caption{Number of edits over budget}
  \label{subfig:tajectory3}
  \end{subfigure} 
  \caption{Training trajectories for the proposed model and the effect of training schedules are visualized.}
  \label{fig:training_trajectory}
\end{figure*}
\subsection{Reproduciblity}
Table~\ref{tab:hyperparam} lists hyperparameters that were used for the different networks we experimented with.

\begin{table*}
\small
\centering
    \begin{tabular}{c||ccc}
    \multicolumn{4}{c}{\textbf{.}} \\

    \hline
                  & Synthetic Graphs  & Road Networks  & Social Networks \\
    \hline\hline
    Budget, $B$  &  400    &   400    & $0.2 |V|$  \\
    Number of groups, $|G|$  &  2    &  3     &  2 \\
    MRP event horizon, $T$  &   10   &  10     & 3  \\
    Discount Factor, $\gamma$  & 0.99      &     0.99  & 0.7  \\
    Budget Lagrangian learning coefficient, $\mu_1$  &  $0.1$    &    $1$   & $10000$  \\
    Equity Lagrangian learning coefficient, $\mu_2$  &  $1\times 10^{-6}$    &  $1\times 10^{-6}$     & $1\times 10^{-6}$  \\
    Number of epochs, $E$  &   700   &  500     & 1000  \\
    Equity constraint schedule  &   $E>0$   &  $E>0$     & $E>0$  \\
    Budget constraint schedule  &   $E>100$   &  $E>100$     & $E>100$   \\
    Graph augmentation discretization schedule, $\tau \to \nu . \tau$  &   $E>200$   &  $E>200$     & $E>300$  \\
    Temperature attenuation factor, $\nu$  &  0.999    &    0.995   & 0.995  \\
\end{tabular}
\caption{Hyperparameters used for different networks.}
\label{tab:hyperparam}
\end{table*}